\documentclass{article}
\usepackage{arxiv}

\usepackage[utf8]{inputenc} 
\usepackage[T1]{fontenc}    
\usepackage{hyperref}       
\usepackage{url}            
\usepackage{booktabs}       
\usepackage{amsfonts}       
\usepackage{nicefrac}       
\usepackage{microtype}      

\usepackage{amssymb,amsmath,amsfonts}
\usepackage{amsmath}
\usepackage{graphicx}
\usepackage{times}
 \usepackage{ifpdf}
 \usepackage{amsthm}
\usepackage{algorithm}
\usepackage[algo2e, ruled, vlined]{algorithm2e}
 \usepackage{url}
\newcommand{\ignore}[1]{}
\newcommand{\notinproc}[1]{#1}
\newcommand{\onlyinproc}[1]{}

\def\E{\textsf{E}}
\def\Exp{\textsf{Exp}}
\def\Var{\textsf{Var}}
\def\MSE{\textsf{MSE}}
\def\Cov{\textsf{Cov}}
\def\Bias{\textsf{Bias}}
\def\Erlang{\textsf{Erlang}}
\def\tail{\textsf{tail}}
\def\order{\textsf{order}}
\def\ekey{\textsf{key}}
\def\eval{\textsf{val}}
\def\bnu{\boldsymbol{\nu}}
\def\Zipf{\textsf{Zipf}}
\def\poly{\mathop{\mathrm{poly}}}
\def\polylog{\mathop{\mathrm{polylog}}}
\newtheorem{thm}{Theorem}[section]
\newtheorem{theorem}{Theorem}[section]

\newtheorem{lemma}[thm]{Lemma}

\newtheorem{definition}[thm]{Definition}

\newtheorem{corollary}[thm]{ Corollary}

\usepackage{xcolor} 

\SetCommentSty{mycommfont}

\title{WOR and $p$'s: \\Sketches for $\ell_p$-Sampling Without Replacement}

\author{
  Edith Cohen\\ 
    Google Research\\
    Tel Aviv University\\
    \texttt{edith@cohenwang.com} \\
  \And
  Rasmus Pagh\\
  IT University of Copenhagen\\
  BARC\\
  Google Research\\
  \texttt{pagh@itu.dk}
  \And
  David P. Woodruff\\
  CMU\\
  \texttt{dwoodruf@cs.cmu.edu}
 }
\date{}
\begin{document}
\maketitle

\begin{abstract}
Weighted sampling is a fundamental tool in data analysis and machine learning pipelines. Samples are used for efficient estimation of statistics or as sparse representations of the data.  When weight distributions are skewed, as is often the case in practice, without-replacement (WOR) sampling is much more effective than with-replacement (WR) sampling: it provides a broader representation and higher accuracy for the same number of samples.
We design novel composable sketches for WOR {\em $\ell_p$ sampling}, weighted sampling of keys according to a power $p\in[0,2]$ of their frequency (or for signed data, sum of updates). Our sketches have size that grows only linearly with the sample size. 
Our design is simple and practical, despite intricate analysis, and based on off-the-shelf use of widely implemented heavy hitters sketches such as \texttt{CountSketch}. Our method is the first to provide WOR sampling in the important regime of $p>1$ and the first to handle signed updates for $p>0$.
\end{abstract}


\section{Introduction}

Weighted random sampling is a fundamental tool that is pervasive in
machine learning and data analysis pipelines.   
A sample serves as a sparse summary of the data and provides efficient estimation of statistics and aggregates.

 We consider data $\mathcal{E}$ presented as {\em elements} in the form of key value pairs $e=(e.\ekey,e.\eval)$.  We operate with respect to the {\em aggregated} form of keys and their {\em frequencies} $\nu_x := \sum_{e\mid e.\ekey=x} e.\eval$, defined as the sum of values of elements with key $x$.
Examples of such data sets are stochastic gradient updates (keys are parameters and element values are signed and the aggregated form is the combined gradient), search (keys are queries, elements have unit values, and the aggregated form are query-frequency pairs), or training examples for language models (keys are co-occurring terms).  

   
The data is commonly distributed across servers or devices or is streamed and the number of distinct keys is very large.  In this scenario it is beneficial to perform computations without explicitly producing a table of key-frequency pairs, as this requires storage or communication that grows linearly with the number of keys. Instead, we use {\em composable sketches} which are data structures that support
(i) {\em processing a new element} $e$:  Computing a sketch of $\mathcal{E}\cup \{e\}$ from a sketch of $\mathcal{E}$ and $e$ (ii) {\em merging}: Computing a sketch of  $\mathcal{E}_1 \cup \mathcal{E}_2$ from sketches of each $\mathcal{E}_i$ and (iii) are such that the desired output can be produced from the sketch. Composability facilitates parallel, distributed, or streaming computation. We aim to design sketches of small size, because the sketch size determines the storage and communication requirements. For sampling, we aim for the sketch size to be not much larger than the desired sample size.

\paragraph{The case for $p$'s:} Aggregation and statistics of functions of the frequencies are essential for resource allocation, planning, and management of large scale systems across application areas. The need for efficiency prompted rich theoretical and applied work on streaming and sketching methods that spanned decades~\cite{MisraGries:1982,GM:sigmod98,ams99,EV:ATAP02,hv:imc03,DLT:sigcomm03,CormodeMuthu:2005,LiuMVSB:sigcomm2016,LiuBEKBFS:Sigcomm2019}.
We study {\em $\ell_p$ sampling}, weighted sampling of keys with respect to a power $p$ of their frequency $\nu^p_x$. 
These samples support
estimates of frequency statistics of the general form $\sum_x f(\nu_x) L_x$ for functions of frequency $f$ and constitute sparse representations of the data. Low powers ($p<1$) are used to mitigate frequent keys and obtain a better resolution of the tail whereas higher powers ($p>1$)  emphasize more frequent keys.
Moreover, recent work suggests that on realistic distributions, $\ell_p$ samples for $p\in [0,2]$ provide accurate estimates for a surprisingly broad set of tasks \cite{CohenGeriPagh:ICML2020}. 
  
Sampling is at the heart of stochastic optimization. When training data is distributed~\cite{pmlr-v54-mcmahan17a}, sampling can facilitate efficient
example selection for training and efficient communication of gradient updates of model parameters. 
 Training examples are commonly weighted by a function of their frequency: Language models~\cite{Mikolov:NIPS13,PenningtonSM14:EMNLP2014} use low powers $p<1$ of frequency to mitigate the impact of frequent examples.  More generally, the function of frequency can be adjusted in the course of training to shift focus to rarer and harder examples as training progresses~\cite{curriculumlearning:ICML2009}. A sample of examples can be used to produce stochastic gradients or evaluate loss on domains of examples (expressed as frequency statistics). 
 In distributed learning, the communication of dense gradient updates can be a bottleneck, prompting the development of methods that sparsify communication while retaining accuracy~\cite{pmlr-v54-mcmahan17a,QSGD:nips2017,StichCJ:NIPS2018,DBLP:conf/nips/IvkinRUBSA19}.  Weighted sampling by the $p$-th powers of magnitudes complements existing methods that sparsify using heavy hitters (or other methods, e.g., sparsify randomly), provides adjustable emphasis to larger magnitudes, and retains sparsity as updates are composed.

\paragraph{The case for WOR:}
Weighted sampling is classically considered with (WR) or without (WOR) replacement. We study here the WOR setting.
The benefits of WOR sampling were noted in very early work \cite{HT52,HatleyRao:1962,Tille:book} and are becoming more apparent with modern applications and the typical skewed distributions of massive datasets.
WOR sampling provides a broader representation and more accurate estimates, with tail norms replacing full norms in error bounds. 
Figure~\ref{effectivesize:plot} illustrates these benefits of WOR for Zipfian distributions with $\ell_1$ sampling (weighted by frequencies) and $\ell_2$ sampling (weighted by the squares of frequencies).  We can see that WR samples have a smaller effective sample size than WOR (due to high multiplicity of heavy keys) and that while both WR and WOR well-approximate the frequency distribution on heavy keys, WOR provides a much better approximation of the tail. 
\begin{figure}[ht]
\centering
\centerline{
\includegraphics[width=0.3\textwidth]{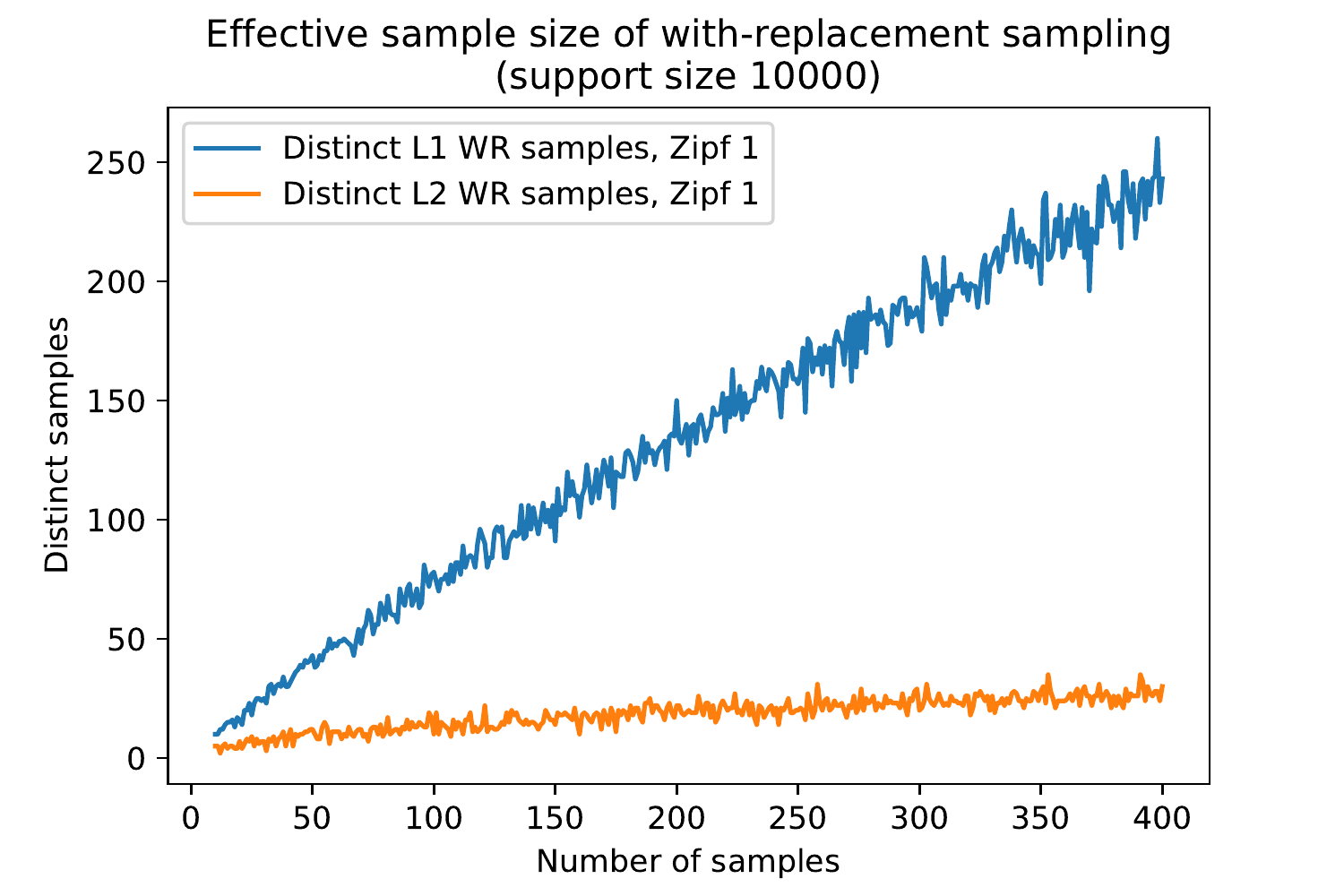}
\includegraphics[width=0.3\textwidth]{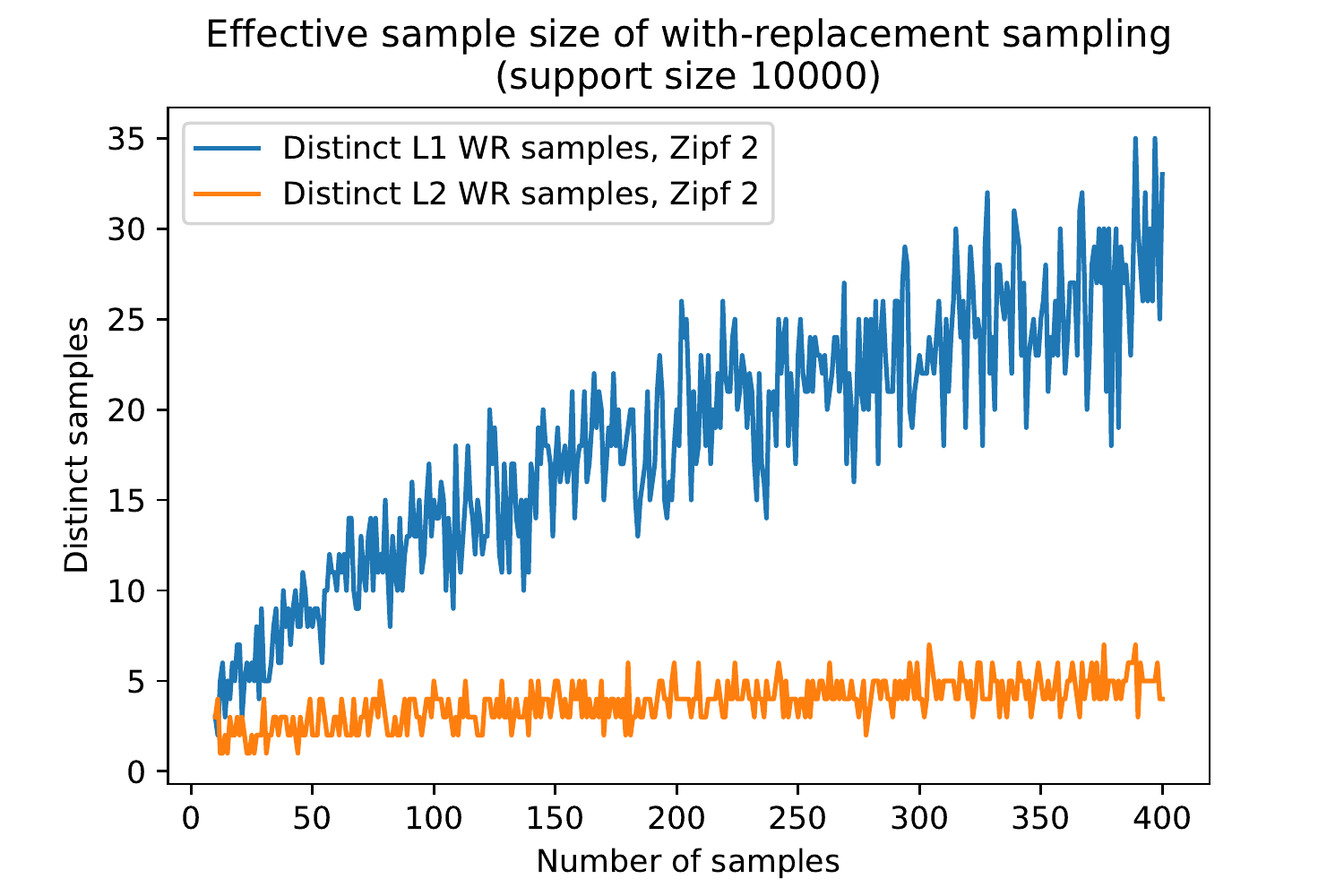}
\includegraphics[width=0.3\textwidth]{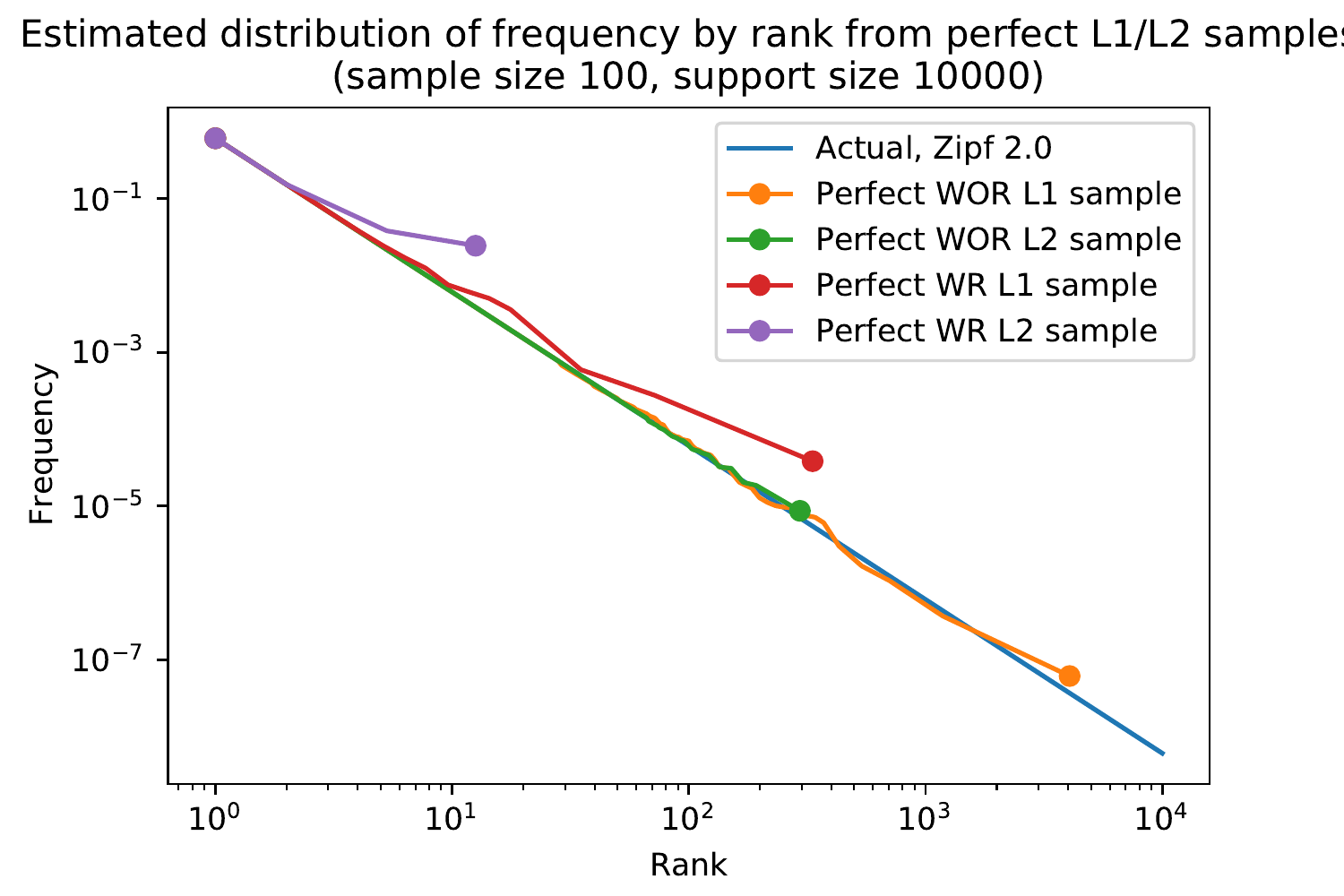}
}
\caption{WOR vs WR. Left and middle: Effective vs actual sample size $\Zipf[\alpha=1]$ and $\Zipf[\alpha=2]$, with each point reflecting a single sample. Right: Estimates of the frequency distribution $\Zipf[\alpha=2]$.
}
\label{effectivesize:plot}
\end{figure}

\paragraph{Related work.}
The sampling literature offers many WOR sampling schemes for aggregated data:~\cite{Rosen1972:successive,Tille:book,Cha82,Rosen1997a,Ohlsson_SPS:1998,DLT:jacm07,bottomk07:ds,bottomk07:est,varopt_full:CDKLT10}.  A particularly appealing technique is bottom-$k$ (order) sampling, where weights are scaled by random variables and the sample is the set of keys with top-$k$ transformed values~\cite{Rosen1997a,Ohlsson_SPS:1998,DLT:jacm07,bottomk07:ds,bottomk07:est}. 
There is also a large body of work on sketches for sampling 
unaggregated data by functions of frequency. There are two primary approaches.
The first approach involves transforming data elements so that a bottom-$k$ sample by function of frequency is converted to an easier 
problem of finding the top-$k$ keys sorted according to the {\em maximum} value of an element with the key. 
This approach yields WOR distinct ($\ell_0$) sampling~\cite{Knuth2f}, $\ell_1$ sampling~\cite{GM:sigmod98,CCD:sigmetrics12}, and sampling with respect to any concave sublinear functions of frequency (including $\ell_p$ sampling for $p\leq 1$)~\cite{freqCap:KDD2015,CohenGeri:NeurIPS2019}). These sketches 
work with non-negative element values but only provide limited support for negative updates~\cite{GemullaLH:vldb06,CCD:sigmetrics12}. 
The second approach performs WR $\ell_p$ sampling for $p\in [0,2]$ using sketches that are random projections~\cite{indyk:stable,fis08,abiw09,jw09,MoWo:SODA2010,AndoniKO:FOCS11,Jowhari:Saglam:Tardos:11,JayaramW:Focs2018}. The methods support signed updates but were not adapted to WOR sampling.
For $p>2$, a classic lower bound~\cite{ams99,b04} establishes that sketches of size polynomial in the number of distinct keys are required for worst case frequency distributions. This task has also been studied in distributed settings \cite{ctw16,jstw19}; \cite{jstw19} observes the importance of WOR in that
setting though does not allow for updates to element values.

\paragraph{Contributions:} We present WORp:  A method for WOR $\ell_p$ sampling for $p\in [0,2]$ via composable sketches of size that grows linearly with the sample size.  WORp is simple and practical and uses a bottom-$k$ transform to  reduce sampling to {\em (residual) heavy hitters} (rHH) computation on the transformed data. The technical heart of the paper is establishing that for any set of input frequencies, the keys with top-$k$ transformed frequencies are indeed rHH. In terms of implementation,  WORp only requires an off-the-shelf use of popular (and widely-implemented) HH sketches~\cite{MisraGries:1982,MM:vldb2002,ccf:icalp2002,CormodeMuthu:2005,spacesaving:ICDT2005,BerindeCIS:PODS2009}.
WORp is the first WOR $\ell_p$ sampling method (that uses sample-sized sketches) for the regime $p\in (1,2]$ and the first to fully support negative updates for $p\in (0,2]$.   As a bonus, we include practical optimizations (that preserve the theoretical guarantees) and perform experiments that demonstrate both the practicality and accuracy of WORp.\footnote{Code for the experiments is provided in the following Colab notebook 
\url{https://colab.research.google.com/drive/1JcNokk_KdQz1gfUarKUcxhnoo9CUL-fZ?usp=sharing} 
}

In addition to the above, we show that perhaps surprisingly, it is possible to obtain a WOR $\ell_p$-sample of a set of $k$ indices, for any $ p \in [0,2]$, with variation distance at most $\frac{1}{\poly(n)}$ to a true WOR $\ell_p$-sample, and using only $k \cdot \poly(\log n)$ bits of memory. Our variation distance is extremely small, and cannot be detected by any polynomial time algorithm. This makes it applicable in settings for which privacy may be a concern; indeed, this shows that no polynomial time algorithm can learn anything from the sampled output other than what follows from a simulator who outputs a WOR $\ell_p$-sample from the actual (variation distance $0$) distribution. Finally, for $p \in (0,2)$, we show that the memory of our algorithm is optimal up to an $O(\log^2 \log n)$ factor.

\section{Preliminaries}

A dataset $\mathcal{E}$ consists of {\em data elements} 
that are key value pairs $e=(e.\ekey,e.\eval)$.
 The frequency of a key $x$,
  denoted $\nu_x := \sum_{e\mid e.\ekey=x} e.\eval$, is the sum of values of elements with key $x$.   We use the notation $\bnu$ for a vector of frequencies of keys.  

For a function $f$ and vector $\boldsymbol{w}$, we denote the vector with entries $f(w_x)$ by $f(\boldsymbol{w})$. In particular, $\boldsymbol{w}^p$ is the vector with entries $w^p_x$ that are the $p$-th powers of the entries of $\boldsymbol{w}$.
For vector $\boldsymbol{w}\in \Re^n$ and index $i$, we denote by $w_{(i)}$ the value of the entry with the $i$-th largest magnitude in $\boldsymbol{w}$. We denote by $\order(\boldsymbol{w})$ the permutation of the indices $[n] = \{1, 2, \ldots, n\}$ that corresponds to decreasing order of entries by magnitude.
For $k\geq 1$, we denote by
$\tail_{k}(\boldsymbol{w})$ the vector with the $k$ entries with largest magnitudes removed (or replaced with $0$).

In the remainder of the section we review ingredients that we build on:  bottom-$k$ sampling,  implementing a bottom-$k$ transform on unaggregated data, and composable sketch structures for residual heavy hitters (rHH).

\subsection{\texorpdfstring{Bottom-$k$}{Bottom-k} sampling (ppswor and priority)} 

Bottom-$k$ sampling (also called order sampling~\cite{Rosen1997a}) is a family of without-replacement
weighted sampling schemes of a set $\{(x,w_x)\}$ of key and weight pairs.  The weights $(x,w_x)$ are transformed via
independent random maps
$w^T_x \gets \frac{w_x}{r_x}\enspace$,  where $r_x\sim \mathcal{D}$ are i.i.d. from some distribution $\mathcal{D}$.
The sample includes the pairs $(x,w_x)$ for keys $x$ that are top-$k$ by transformed magnitudes $|\boldsymbol{w^T}|$~\cite{Ohlsson_SPS:1990,Rosen1997a,DLT:jacm07,ECohen6f,BRODER:sequences97,bottomk07:ds,bottomk:VLDB2008}~\footnote{Historically, the term bottom-$k$ is due to analogous use of $1/w_x^T$, but here we find it more convenient to work with "top-$k$" }. For estimation tasks, we also include a {\em threshold}
$\tau := |w^T_{(k+1)}|$, the
$(k+1)$-st largest magnitude of transformed weights.
Bottom-$k$ schemes differ by the choice of distribution $\mathcal{D}$. Two popular choices are
Probability Proportional to Size WithOut Replacement (ppswor)~\cite{Rosen1997a,ECohen6f,bottomk:VLDB2008}
via the exponential distribution $\mathcal{D} \leftarrow  \Exp[1]$ and Priority (Sequential Poisson) sampling
\cite{Ohlsson_SPS:1990,DLT:jacm07} via the uniform distribution $\mathcal{D} \leftarrow  U[0,1]$.  Ppswor is equivalent to a  weighted sampling 
process~\cite{Rosen1972:successive} where keys are drawn successively
(without replacement) with probability proportional to their
weight. Priority sampling mimics a pure Probability Proportional
to Size (pps) sampling, where sampling probabilities are  proportional to weights (but truncated to be at most $1$).

\paragraph{Estimating statistics from a \texorpdfstring{Bottom-$k$}{Bottom-k} sample}
Bottom-$k$ samples provide us with unbiased inverse-probability~\cite{HT52}  per-key estimates on $f(w_x)$, where $f$ is a function applied to the weight~\cite{dlt:pods05,bottomk07:est,freqCapfill:TALG2018,multiobjective:2015}):
\begin{equation}\label{perkey:eq}
\widehat{f(w_x)}:=
\begin{cases}
\frac{f(w_x)}{\Pr_{r\sim\mathcal{D}}[r\leq |w_x|/\tau]} & \text{ if } x \in S\\
0 & \text{ if } x \notin S 
\end{cases} \enspace . 
\end{equation}
These estimates can be used to sparsify a vector $f(\boldsymbol{w})$ to $k$ entries or to estimate sum statistics of the general form:
\begin{equation} \label{sumstat:eq}
    \sum_x f(w_x) L_x
\end{equation}
using the unbiased estimate
\[
\widehat{\sum_x f(w_x) L_x} := \sum_x \widehat{f(w_x)} L_x = \sum_{x\in S} \widehat{f(w_x)} L_x\enspace .
\]
The quality of estimates depends on $f$ and $L$.  We measure the quality of these unbiased estimates by the sum over keys of the per-key variance. With both ppswor and priority sampling and $f(w):=w$, the sum is bounded by a respective one for pps with replacement. The per-key variance bound is
\begin{equation} \label{keyvarbound:eq}
 \Var[\widehat{w_x}] \leq  \frac{1}{k-1} w_x \|\boldsymbol{w} \|_1\enspace 
\end{equation}
and the respective sum by $\sum_x \Var[\widehat{w_x}] \leq \frac{1}{k-1} \|\boldsymbol{w} \|^2_1$.  This can be tightened to
$\Var[\widehat{w_x}] \leq  \min\{O(\frac{1}{k}) w_x \|\tail_k(\boldsymbol{w})\|_1, \exp\left(-O(k) \frac{w_x}{\|\tail_k(\boldsymbol{w})\|_1}\right) w_x^2 \}\enspace$  
and respective bound on the sum of $O(\|\tail_k(\boldsymbol{w})\|^2_1/k)$.
For skewed distributions,
$\|\tail_k(\boldsymbol{w})\|^2_1 \ll \|\boldsymbol{w} \|^2_1$ and hence WOR sampling is beneficial. 
Conveniently, bottom-$k$ estimates for different keys $x_1\not= x_2$ have non-positive correlations $\Cov[\widehat{w_{x_1}},\widehat{w_{x_2}}]\leq 0$, so the variance of sum estimates is at most the respective weighted sum of per-key variance. Note that the per-key variance for a function of weight is
$\Var[\widehat{f(w_x)}] = \frac{f(w_x)^2}{w_x^2} \Var[\widehat{w_x}]$.
WOR (and WR) estimates are more accurate (in terms of normalized variance sum) when the sampling is weighted by $f(\boldsymbol{w})$.
\ignore{
We will consider here 
bottom-$k$ sampling with respect to approximate weight.  is approximate sampling and the weight that is performed with respect to approximate $w_x$ (the entries $w^T_x$ are noised by at most $(1\pm \epsilon)$ factor).   Standard analysis shows that this increases the variance by a corresponding factor (see e.g.~\cite{multiw:VLDB2009,multiobjective:2015}). Concentration bounds are also similarly minimally affected. 
Finally, we also consider estimates where $w_x$ is also known within a relative error, which introduces an error into the numerator of~\eqref{perkey:eq}.  The estimate then becomes biased (if the noise is) but NRMSE increases by at most $O(\epsilon)$. Concentration bounds apply with the respective additive noise.
}

\subsection{\texorpdfstring{Bottom-$k$}{Bottom-k} sampling by power of frequency}
To perform bottom-$k$ sampling of $\boldsymbol{w}^p$ with distribution $\mathcal{D}$, we draw $r_x\sim \mathcal{D}$, transform the weights $w^T_x \gets w_x^p/r_x$, and 
return the top-$k$ keys in $\boldsymbol{w}^T$.  
This is equivalent to bottom-$k$ sampling the vector $\boldsymbol{w}$ using the distribution
$\mathcal{D}^{1/p}$, that is, draw $r_x \sim \mathcal{D}$, transform the weights
  \begin{equation} \label{bottomkptransform:eq}
 w^*_x \gets \frac{w_x}{r_x^{1/p}}
  \end{equation}
and return the top-$k$ keys according to $\boldsymbol{w}^*$. 
Equivalence is because $(w^*_x)^p = \left( \frac{w_x}{r_x^{1/p}}\right)^p = \frac{w_x^p}{r_x} = w^T_x$ and $f(x)=x^p$ is a monotone increasing and hence $\order(\boldsymbol{w}^*)=\order(\boldsymbol{w}^T)$.  We denote the distribution of $\boldsymbol{w}^*$ obtained from the bottom-$k$ transform \eqref{bottomkptransform:eq} as $p$-$\mathcal{D}[\boldsymbol{w}]$ and specifically, $p$-ppswor$[\boldsymbol{w}]$ when $\mathcal{D}=\Exp[1]$ and $p$-priority$[\boldsymbol{w}]$ when $\mathcal{D}=U[0,1]$. We use the term $p$-ppswor for bottom-$k$ sampling by $\Exp^{1/p}$.
 
  The linear transform \eqref{bottomkptransform:eq} can be efficiently performed over unaggregated data by using a random hash to represent $r_x$ for keys $x$ and then
 locally generating an output element 
 for each input element
 \begin{equation} \label{ptrasform:eq}
    (e.\ekey,e.\eval) \rightarrow (e.\ekey,e.\eval/r_{e.\ekey}^{1/p})
 \end{equation}
  The task of sampling by $p$-th power of frequency $\boldsymbol{\nu}^p$ is replaced by the task of top-$k$ by frequency $\nu^*_x := \sum_{e\in \mathcal{E}^* \mid e.\ekey=x} e.\eval$ on the respective output dataset $\mathcal{E}^*$, noting that
  $\nu^*_x = \nu_x/r_x^{1/p}$.  Therefore,  the top-$k$ keys in $\boldsymbol{\nu}^*$ are a bottom-$k$ sample according to $\mathcal{D}$ of  $\boldsymbol{\nu}^p$.
Note that we can approximate the input frequency $\nu'_x$ of a key $x$ from an approximate output frequency  $\widehat{\nu^*_x}$ using the hash $r_x$. Note that relative error is preserved:
\begin{equation} \label{out2in:eq}
\nu'_x \gets \widehat{\nu^*_x} r_x^{1/p}\enspace .
\end{equation}
  This per-element scaling was proposed in the {\em precision sampling} framework of
Andoni et al. \cite{AndoniKO:FOCS11} and inspired by a technique for frequency moment estimation using stable distributions~\cite{indyk:stable}.
 
  Generally, finding the top-$k$ frequencies is a task that requires 
large sketch structures of size linear in the number of keys. However, \cite{AndoniKO:FOCS11} showed that when the frequencies are drawn from $p$-priority$[\boldsymbol{w}]$ (applied to arbitrary $\boldsymbol{w}$) and $p\leq 2$ then the top-$1$ value is likely to be an $\ell_2$ {\em heavy hitter}. Here we refine the analysis and use the more subtle notion of {\em residual heavy hitters} ~\cite{BerindeCIS:PODS2009}.  We show that the top-$k$ output frequencies in $\boldsymbol{w}^*\sim p$-ppswor$[\boldsymbol{w}]$ are very likely to be $\ell_q$  residual heavy hitters (when $q\geq p$) and can be found with a sketch of size $\tilde{O}(k)$.

\subsection{Residual heavy hitters (rHH)}

 An entry in a weight vector $\boldsymbol{w}$ is called an {\em
   $\varepsilon$- heavy hitter} if $w_x \geq \varepsilon \sum_y w_y$.
 A heavy hitter with respect to a function $f$ is defined
 as a key with $f(w_x) \geq \varepsilon \sum_y f(w_y)$.   When $f(w)=w^q$, we
 refer to a key as an $\ell_q$ heavy hitter.  For $k\geq 1$ and
 $\psi>0$, a vector
 $\boldsymbol{w}$ has  {\em
    $(k,\psi)$ residual heavy hitters}~\cite{BerindeCIS:PODS2009} when
  the top-$k$ keys are  ``heavy" with respect to the tail of the
  frequencies starting at the $(k+1)$-st most frequent key, that is, $\forall i\leq k,\  w_{(i)} \geq \frac{\psi}{k}
  \|\tail_{k}(\boldsymbol{w})\|_1$.
This is equivalent to 
$\frac{\|\tail_{k}(\boldsymbol{w})\|_1}{w_{(k)}} \leq \frac{k}{\psi}$.
We say that $\boldsymbol{w}$ has rHH with respect to a function $f$ if $f(\boldsymbol{w})$ has rHH.
In particular, $\boldsymbol{w}$ has $\ell_q$ $(k,\psi)$ rHH if
\begin{equation} \label{RHHcond:eq}
\frac{\|\tail_{k}(\boldsymbol{w})\|_q^q}{w^q_{(k)}} \leq \frac{k}{\psi} \enspace .
\end{equation}

\ignore{
Similarly, we can define rHH with respect to a function of frequencey $f$.
  .  This definition includes keys that may not be heavy hitters in the traditional sense.  For a vector $\boldsymbol{w}$ we will use 
\begin{definition}
For $k\geq 1$ and  $\psi>0$, the frequencies $\boldsymbol{w}$ (sorted in non-increasing order) have {\em $\ell_p$ $(k,\psi)$ residual heavy hitters} if
\[\forall i\leq k,\  w_i^p \geq \frac{\psi}{k} \|\tail_{k}(\boldsymbol{w})\|_p^p = \frac{\psi}{k} \sum_{j>k} w_j^p \enspace .\]
\end{definition}  
An equivalent condition is the following
\begin{equation} \label{RHHcond:eq}
\psi^{-1} \geq \frac{1}{k}\frac{\sum_{j=k+1}^n w^p_j}{w_k^p} \enspace .
\end{equation}
This is equivalent because the ratio bound for the $k$-th most frequent
key applies to all $k$ most frequent keys.  This form is interesting
to us because in order to determine existence of rHH it suffices to
bound the ratio $\frac{\sum_{j=k+1}^n w^p_j}{w_k^p}$  from above.
}

\SetKwFunction{Rinit}{Initialize} 
\SetKwFunction{Rprocess}{Process} 
\SetKwFunction{Rmerge}{Merge} 
\SetKwFunction{Rest}{Est} 


 Popular composable HH sketches were adopted to rHH and include (see Table~\ref{rHH:table}): (i)~$\ell_1$ sketches designed for positive data elements. A deterministic counter-based variety~\cite{MisraGries:1982,MM:vldb2002,spacesaving:ICDT2005} with rHH adaptation~\cite{BerindeCIS:PODS2009} and the randomized \texttt{CountMin} sketch~\cite{CormodeMuthu:2005}. (ii)~$\ell_2$ sketches designed for signed data elements, notably \texttt{CountSketch}~\cite{ccf:icalp2002} with rHH analysis in~\cite{Jowhari:Saglam:Tardos:11}.
 With these designs, a sketch for $\ell_q$ $(k,\psi)$-rHH provides estimates 
$\widehat{\nu_x}$ for all keys $x$ with error bound:
\begin{equation} \label{rHHerror:eq}
    \|\widehat{\bnu}- \bnu \|^q_{\infty} \leq \frac{\psi}{k}\|\tail_{k}(\boldsymbol{\nu})\|_q^q\enspace .
\end{equation}
With randomized sketches, the error bound \eqref{rHHerror:eq} is guaranteed with some probability $1-\delta$.
\texttt{CountSketch} has the advantages of capturing top keys that are $\ell_2$ but not $\ell_1$ heavy hitters and supports signed data, but otherwise provides lower accuracy than $\ell_1$ sketches for the same sketch size. Methods also vary in supported key domains: counters natively work with key strings whereas randomized sketches work for keys from $[n]$ 
(see Appendix~\ref{rHHmore:sec} for a further discussion).
  We use these sketches off-the-shelf through the following operations: 
\begin{itemize}
    \item 
$R.\Rinit(k,\psi,\delta)$: Initialize a sketch structure
\item
$\Rmerge(R_1,R_2)$: Merge two sketches with the same parameters and internal randomization
\item
$R.\Rprocess(e)$: process a data element $e$
\item
$R.\Rest(x)$: Return an estimate of the frequency of a key $x$ .
\end{itemize}

 \begin{table}[H]
   \center
 \begin{tabular}{l|ll}
Sketch ($\ell_q$, sign) &  Size  & $\|\widehat{\bnu}- \bnu \|^q_{\infty} \leq$\\
   \hline
\texttt{Counters} ($\ell_1$, $+$) \cite{BerindeCIS:PODS2009}  & $O(\frac{k}{\psi})$ &
                                                         $\frac{\psi}{k}
                                                         \|\tail_{k}(\boldsymbol{\nu})\|_1$
   \\
\texttt{CountSketch} ($\ell_2$,$\pm$)  \cite{ccf:icalp2002}   & $O(\frac{k}{\psi}\log \frac{n}{\delta})$ & $\frac{\psi}{k}
                                                         \|\tail_{k}(\boldsymbol{\nu})\|_2^2$
 \end{tabular}
 \caption{Sketches for $\ell_q$ $(k,\psi)$ rHH.}\label{rHH:table} 
\end{table}


\section{WORp Overview}

Our WORp methods apply a $p$-ppswor transform to data elements  \eqref{ptrasform:eq} and (for $q\geq p$) compute an $\ell_q$ $(k,\psi)$-rHH sketch of the output elements. The rHH sketch is used to produce a sample of $k$ keys.

We would like to set $\psi$ to be low enough so that for any input frequencies $\bnu$, the top-$k$ keys by transformed frequencies $\bnu^*\sim \text{$p$-ppswor}[\bnu]$ are rHH (satisfy condition \eqref{RHHcond:eq}) with probability at least $1-\delta$.  
We refine this desiderata to be conditioned on the permutation of keys in $\order(\bnu^*)$.  This conditioning turns out not to further constrain  $\psi$ and allows us to provide the success probability uniformly for any potential $k$-sample.  Since our sketch size grows inversely with $\psi$ (see Table~\ref{rHH:table}), we want to use the maximum value that works. We will be guided by the following: 
{\small
 \begin{equation}
     \Psi_{n,k,\rho=q/p}(\delta) := \sup 
     \left\{ \psi \mid  \forall \boldsymbol{w}\in \Re^n, \pi\in S^n \Pr_{\boldsymbol{w}^*\sim\text{$p$-ppswor}[\boldsymbol{w}] \mid \order(\boldsymbol{w}^*) =\pi} \left[
          k \frac{|w^*_{(k)}|^q}{\|\tail_{k}(\boldsymbol{w}^*)\|_q^q} \leq  \psi\right] 
\leq \delta 
     \right\},
 \end{equation}
 }
 where $S^n$ denotes the set of permutations of $[n]$.
 If we set the rHH sketch parameter to $\psi \gets \varepsilon \Psi_{n,k,\rho}$ then using \eqref{rHHerror:eq}, with probability at least $1-\delta$,
\begin{equation} \label{usePsi:eq}
  \|\widehat{\bnu^*}- \bnu^* \|^q_{\infty} \leq \frac{\psi}{k} \|\tail_k(\boldsymbol{\nu}^*)\|_q^q = \varepsilon \frac{\Psi_{n,k,\rho}(\lambda)}{k}  \|\tail_k(\boldsymbol{\nu}^*)\|_q^q   \leq \varepsilon |\nu^*_{(k)}|^q\ .
\end{equation}

We establish the following lower bounds on  $\Psi_{n,k,\rho}(\delta)$:
\begin{theorem} \label{rHHprop:thm}
  There is a universal constant $C>0$ such that for 
all $n$, $k>1$, and $\rho=q/p$
\begin{align}
\text{For $\rho=1$:}\, & \Psi_{n,k,\rho}(3e^{-k}) \geq \frac{1}{C \ln \frac{n}{k}} \\
\text{For $\rho >1$:}\, & \Psi_{n,k,\rho}(3e^{-k}) \geq \frac{1}{C} \max\{ \rho-1, \frac{1}{\ln \frac{n}{k})} \}\enspace . \end{align}
\end{theorem}
To set sketch parameters in implementations, we
approximate $\Psi$ using simulations of what we establish to be a ``worst case" frequency distribution.
For this we use a specification of a ``worst-case" set of frequencies as part of the proof of Theorem~\ref{rHHprop:thm} (See Appendix~\ref{psiapprox:sec}).
From simulations we obtain that very small values of $C < 2$ suffice for $\delta = 0.01$, $\rho\in\{1,2\}$, and $k\geq 10$.


\SetKwFunction{keyhash}{KeyHash} 

We analyse a few WORp variants.  The first we consider returns an exact $p$-ppswor sample, including exact frequencies of keys, using two passes.  We then consider a variant that returns an approximate $p$-ppswor sample in a single pass.  The two-pass method uses smaller rHH sketches and efficiently works with keys that are arbitrary strings.  

We also provide another rHH-based technique that provides a guaranteed very small variation distance on $k$-tuples in a single pass.

  \section{Two-pass WORp} \label{twopass:sec}
  
  \begin{table}
      \centering
 {\small
\begin{tabular}{l | l l | l}
& \multicolumn{2}{c}{Sketch size} & \\
sign, $p$ & words & key strings & $\Pr[\text{success}]$\\
\hline
$\pm$, $p<2$ & $O(k\log n)$ & $O(k)$ & $(1-\frac{1}{\text{poly}(n)})(1-3e^{-k})$ \\
$\pm$, $p=2$ & $O(k\log^2 n)$ & $O(k)$ & $(1-\frac{1}{\text{poly}(n)})(1-3e^{-k})$ \\
$+$, $p<1$ & $O(k)$ & $O(k)$ & $1-3e^{-k}$ \\
$+$, $p=1$ & $O(k\log n)$ & $O(k)$ & $1-3e^{-k}$ 
\end{tabular} \\
}     
      \caption{Two-pass ppswor sampling of $k$ keys according to $\bnu^p$. Sketch size (memory words and number of stored key strings). For $p\in (0,2]$ and signed ($\pm$) or positive ($+$) value elements.}
      \label{twopass:tab}
  \end{table}
We design a two-pass method for ppswor sampling according to $\bnu^p$ for $p\in (0,2]$ (Equivalently, a $p$-ppswor sample according to $\bnu$): 

\begin{trivlist}
 \item[$\bullet$] {\bf Pass I:}  
We compute an $\ell_q$ $(k+1,\psi)$-rHH sketch $R$ of the transformed data elements
\begin{equation} \label{oelem:eq}
(\keyhash(e.\ekey),e.\eval/r_{e.\ekey}^{1/p})\enspace .
\end{equation}
A hash $\keyhash\rightarrow [n]$ is applied when keys have a large or non-integer domain to facilitate use of \texttt{CountSketch} or reduce storage with \texttt{Counters}.
We set $\psi \gets \frac{1}{3^q} \Psi_{n,k,\rho}(\delta)$. 
\item[$\bullet$] {\bf Pass II:}     
We collect key strings $x$ (if needed) and corresponding exact frequencies $\nu_x$ for keys with the $B k$ largest  $|\widehat{\nu^*_x}|$, where $B$ is a constant (see below) and 
$\widehat{\nu^*_x} := R.\Rest[\keyhash(x)]$ are the estimates of $\nu^*_x$ provided by $R$.
For this purpose we use a composable top-$(B k)$ sketch structure $T$. The size of $T$ is dominated by storing $B k$ key strings. 

\item[$\bullet$] {\bf Producing a $p$-ppswor sample from $T$:}
Compute exact transformed frequencies $\nu^*_x \gets \nu_x r^{1/p}_x$ for stored keys $x\in T$.  
Our sample is the set of key frequency pairs $(x,\nu_x)$ for the top-$k$ stored keys by  $\nu^*_x$.  The threshold $\tau$ is the $(k+1)^{\text{th}}$ largest $\nu^*_x$ over stored keys.  
 \item[$\bullet$] {\bf Estimation:} We compute per-key estimates as in~\eqref{perkey:eq}: Plugging in $\mathcal{D} = \Exp[1]^{1/p}$ for $p$-ppswor, we have $\widehat{f(\nu_x)}:=0$ for $x\not\in S$ and for $x\in S$ is $\widehat{f(\nu_x)} := \frac{f(\nu_x)}{1-\exp\left(- (\frac{\nu_x}{\tau})^p\right)}$.
 \end{trivlist}
  
  We establish that the method returns the $p$-ppswor sample with probability at least $1-\delta$, propose practical optimizations, and analyze the sketch size:
     \begin{theorem} \label{twopass:thm}
   The 2-pass method returns a $p$-ppswor sample of size $k$ according to $\bnu$
   with success probability and composable sketch sizes as 
detailed in Table~\ref{twopass:tab}. 
The success probability is defined to be that of returning the exact top-$k$ keys by transformed frequencies. The bound applies even when  conditioned
on the top-$k$ being a particular $k$-tuple.
\end{theorem}
\begin{proof}
From \eqref{usePsi:eq},  the estimates 
$\widehat{\nu^*_x} = R.\Rest[\keyhash(x)]$ of $\nu^*_x$ are such that: 
\begin{equation} \label{approx:eq}
    \Pr\left[\forall x\in \{e.\ekey \mid e \in \mathcal{E}\} , 
    |\widehat{\nu^*_x}-\nu^*_x| \leq  \frac{1}{3} |\nu^*_{(k+1)}|\right] \geq 1-\delta\enspace .
\end{equation}
   
We set $B$ in the second pass so that the following holds:
\begin{equation}\label{proptop:eq}
\text{The top-$(k+1)$ keys by $\bnu^*$ are a subset of the top-$(B (k+1))$ keys by $\widehat{\bnu^*}$.} 
\end{equation}
Note that for any frequency distribution with rHH, it suffices to store $O(k/\psi)$ keys to have \eqref{proptop:eq}.  We establish \onlyinproc{(see the appendix)} that 
for our particular distributions, a constant $B$ suffices.
\notinproc{
For that we used the following:
\begin{lemma}
A sufficient condition for property~\eqref{proptop:eq} is
that $|\nu^*_{(B (k+1))}| \leq \frac{1}{3} |\nu^*_{((k+1))}|$.
\end{lemma}
\begin{proof}
Note that $|\widehat{\nu^*_x}| \geq \frac{2}{3} |\nu^*_{(k+1)}|$ for keys that are top-$(k+1)$ by $\bnu^*$ and $|\widehat{\nu^*}_{(B(k+1))}|\leq |\nu^*_{(B (k+1))}| + \frac{1}{3} |\nu^*{(k+1)}|$.  Hence $|\widehat{\nu^*_x}| \geq |\widehat{\nu^*}_{(B(k+1)}|$ for all keys that are top-$(k+1)$ by $\bnu^*$.
\end{proof}
We then use Lemma~\ref{domB:lemma} to express a ``worst-case" distribution for the ratio
$\nu^*_{B(k+1)}/\nu^*_(k+1)$ and use the latter (using Corollary~\ref{constB:coro}) to show that the setting of $\Psi(\delta)$ according to 
our proof of Theorem~\ref{rHHprop:thm} (Appendix~\ref{Oproof:sec}-\ref{proofRbound:sec}) implies that the conditions that guarantee the rHH property will also imply a ratio of at most $1/3$ with a constant $B$.
}

Correctness for the final sample follows from property~\eqref{proptop:eq} : $T$ storing the top-$(k+1)$ keys in the data according to $\bnu^*$.
 To conclude the proof of Theorem~\ref{twopass:thm} we need to specify the rHH sketch structure we use.
 From Theorem~\ref{rHHprop:thm} we obtain a lower bound on $\Psi_{n,k,\rho}$ for $\delta = 3e^{-k}$ and we use it to set $\psi$.   For our rHH sketch we use  \texttt{CountSketch} ($q=2$ and supports signed values) or \texttt{Counters} ($q=1$ and positive values). 
 The top two lines in Table~\ref{twopass:tab} are for \texttt{CountSketch} and the next two lines are for \texttt{Counters}.  The rHH sketch sizes follow from $\psi$ and Table~\ref{rHH:table}.
\end{proof}

\notinproc{
\subsection{Practical optimizations}
We propose practical optimizations that improve efficiency without compromising quality guarantees.

The first optimization allows us to store $k'\ll B(k+1)$ keys in the second pass: We always store the top-$(k+1)$ keys by $\widehat{\bnu^*}$ but beyond that only store keys if they satisfy
\begin{equation}\label{condopt:eq}
    \widehat{\nu^*}_x \geq \frac{1}{2} \widehat{\nu^*}_{(k+1)}\enspace ,
\end{equation}
where $\bnu^*$ is with respect to data elements processed by the current sketch.  
We establish correctness:
\begin{lemma}
\begin{enumerate}
    \item 
There is a composable structure that only stores keys that satisfy \eqref{condopt:eq}
and collects exact frequencies for these keys.  
\item
If \eqref{approx:eq} holds, the top-$(k+1)$ keys by  
$\bnu^*$ satisfy~\eqref{condopt:eq} (and hence will be stored in the sketch).
\end{enumerate}
\end{lemma}
\begin{proof}
(i) The structure is a slight modification of a top-$k'$ structure.
Since $\widehat{\nu^*}_{(k+1)}$ can only increase as more distinct keys are processed, the condition \eqref{condopt:eq} only becomes more stringent as we merge sketches and process elements.  Therefore if a key satisfies the condition at some point it would have satisfied the condition when elements with the key were previously processed and therefore we can collect exact frequencies.

(ii) From the uniform approximation \eqref{approx:eq}, we have $\widehat{\nu^*}_{(k+1)} \leq \frac{4}{3} \nu^*_{(k+1)}$.  Let $x$ be the $(k+1)$-th key by $|\bnu^*|$.  Its
estimated transformed frequency is at at least 
$\widehat{\nu^*_x} \geq \frac{2}{3} \nu^*_{(k+1)} \geq  \frac{2}{3} \cdot \frac{3}{4}  \widehat{\nu^*}_{(k+1)} = \frac{1}{2} \widehat{\nu^*}_{(k+1)}$.  Therefore, if we store all keys $x$ with $\widehat{\nu^*_x}\geq \frac{1}{2} \widehat{\nu^*}_{(k+1)}$ we store the top-$(k+1)$ keys by $\nu^*_x$.
\end{proof}

A second optimization allows us to extract a larger effective sample from the sketch with $k'\geq k$ keys.
This can be done when we can certify that the top-$k'$ keys by $\bnu^*$ in the transformed data are stored in the sketch $T$.  Using a larger sample is helpful as it can only improve (in expectation) estimation quality (see e.g., \cite{CK:sigmetrics09,multiw:VLDB2009}).  
To extend this, 
we compute the 
uniform error bound $\nu^*_{(k+1)}/3$ (available because the top-$(k+1)$ keys by $\nu^*$ are stored).  Let $L \gets \min_{x\in T} \widehat{\nu^*}_x$.  We obtain that any key $x$ in the dataset with $\nu^*_x \geq L+ \nu^*_{(k+1)}/3$ must be stored in $T$.  Our  final sample contains these keys with the minimum $\nu^*_x$ in the set used as 
the threshold $\tau$.

}

\section{One-pass WORp} \label{onepass:sec}

Our 1-pass WORp yields a sample of size $k$ that approximates a $p$-ppswor sample of the same size and provides similar guarantees on estimation quality. 
\begin{trivlist}
    \item[$\bullet$] {\bf Sketch:}
    For $q\geq p$ and $\varepsilon \in (0,\frac{1}{3}]$
    Compute an $\ell_q$ $(k+1,\psi)$-rHH sketch $R$ of the transformed data elements \eqref{ptrasform:eq} where
$\psi \gets \varepsilon^q \Psi_{n,k+1,\rho}$.
\item[$\bullet$]    {\bf Produce a sample:}
   Our sample $S$ includes the keys with top-$k$ estimated transformed frequencies  $\widehat{\nu^*_x} := R.\Rest[x]$.
   For each key $x\in S$ we include $(x, \nu'_x)$, where the approximate frequency $\nu'_x\gets \widehat{\nu^*_x} r^{1/p}_x$ is according to \eqref{out2in:eq}.  We include with the sample the threshold $\tau \gets\widehat{\nu^*}_{(k+1)}$.  
\item[$\bullet$] {\bf Estimation:} We treat the sample as a $p$-ppswor sample and compute 
per-key estimates as in~\eqref{perkey:eq}, substituting approximate frequencies $\nu'_x$ for actual frequencies $\nu_x$ of sampled keys and the 1-pass threshold $\widehat{\nu^*}_{(k+1)}$ for the exact $\nu^*_{(k+1)}$.  The estimate is $\widehat{f(\nu_x)}:=0$ for $x\not\in S$ and for $x\in S$ is:
\begin{equation} \label{xppsworest:eq}
\widehat{f(\nu_x)} := \frac{f(\nu'_x)}{1-\exp\left(- (\frac{\nu'_x}{\widehat{\nu^*}_{(k+1)}})^p\right)} =  \frac{ f(\widehat{\nu^*_x} r_x^{1/p})}{1- \exp\left(- r_x (\frac{\widehat{\nu^*_x}}{\widehat{\nu^*}_{(k+1)}})^p\right)}
\end{equation}
\end{trivlist}

We relate the quality of the estimates to those of a perfect $p$-ppswor. Since our 1-pass estimates are biased (unlike the unbiased perfect $p$-ppswor), we consider both bias and variance.  The proof is provided in Appendix~\ref{onepassquality:sec}.
\begin{theorem} \label{onepassquality:thm}
Let $f(w)$ be such that $|f((1+\varepsilon)w)-f(w)| \leq c \varepsilon f(w)$ for some $c>0$ and $\varepsilon\in [0,1/2]$.
Let $\widehat{f(\nu_x)}$ be per-key estimates obtained with a one-pass WORp sample and let $\widehat{f(\nu_x)}'$ be the respective estimates obtained with a (perfect) $p$-ppswor sample.  Then $\Bias[\widehat{f(\nu_x)}] \leq O(\varepsilon) f(\nu_x)$ and
$\MSE[\widehat{f(\nu_x)}] \leq (1+O(\varepsilon)) \Var[\widehat{f(\nu_x)}'] + O(\epsilon) f(\nu_x)^2$.
\end{theorem}
Note that the assumption on $f$ holds for $f(w)=w^p$ with $c=(1.5)^p$.  Also note that the 
bias bound implies a respective contribution to the relative error of $O(\varepsilon)$ on all sum estimates.

\ignore{
********* finish below  also get a better bound for countsketch ******

The claim of the theorem allows for bias in the rHH estimates that is of the order of the error, which can be the case with the \texttt{Counters} rHH sketch. \textttt{CountSketch} estimates are unbiased and respective sample estimates have a lower bias and we can expect a better behavior.

The rHH approximation shifts the threshold $\tau$ by at most $\varepsilon \tau$
The estimate for each key are considered with respect to a fixed randomization of ot
With \textt{Counters} rHH sketches, the bias is as large

  We relate the error to that obtained with the same randomization $r$ on all other keys.  when the threshold and randomization of other keys set according  
There is error from using $f(\nu'_x)$ instead of $f(\nu_x)$.  For $f(w)=w^p$  and $p\leq2$ we retain good relative error.
with some rHH sketches, $\nu'_x$ are biased estimates of $\nu_x$ and therefore our per-key estimate can be biased.

In a practical optimization we use more of the sampled keys as long as long as $\widehat{\bnu^*}$ is large enough compared to the error.

  ******* revise below ***********

  We declare ``fail'' if the estimates are  ``low'' (lower than $\approx (1-\epsilon)\frac{\psi}{k}\|\tail_{k}(\boldsymbol{\nu}^*)\|_q^q$ (we use estimate for the tail).

 }
 
\section{One-pass Total Variation Distance Guarantee}
We provide another 1-pass method, based on the combined use of rHH and known WR perfect $\ell_p$ sampling sketches~\cite{JayaramW:Focs2018} to select a $k$-tuple with a polynomially small total variation (TV) distance from the $k$-tuple distribution of a perfect $p$-ppswor.  The method uses $O(k)$ (for variation distance $2^{-\Theta(k)}$, and $O(k \log n)$ for variation distance $1/n^C$ for an arbitrarily large constant $C > 0$) perfect samplers (each providing a single WR sample) and an rHH sketch.  The perfect samplers are processed in sequence with prior selections ``subtracted" from the linear sketch (using approximate frequencies provided by the rHH sketch) to uncover fresh samples.  As with WORp, exact frequencies of sampled keys can be obtained in a second pass or approximated using larger sketches in a single pass. Details are provided in Appendix~\ref{onepassTD:sec}.
\begin{theorem}\label{thm:perfect}
There is a 1-pass method via composable sketches of size $O(k \polylog(n))$ that returns a $k$-tuple of keys such that the total variation distance from the $k$-tuples produced by a perfect $p$-ppswor sample is at most $1/n^C$ for an arbitrarily large constant $C > 0$. 
The method applies to keys from a domain $[n]$, and signed values with magnitudes and intermediate frequencies that are polynomial in $n$.
\end{theorem}
We also show in Appendix~\ref{onepassTD:sec} that our sketches in Theorem \ref{thm:perfect} use $O(k \cdot \log^2 n (\log \log n)^2)$ bits of memory for $0 < p < 2$, and we prove a matching lower bound on the memory required of any algorithm achieving this guarantee, up to a $(\log \log n)^2$ factor. For $p = 2$ we also show they are of optimal size, up to an $O(\log n)$ factor. 

\section{Experiments}

We implemented 2-pass and 1-pass WORp in Python using \texttt{CountSketch} as our rHH sketch.  
Figure~\ref{perfectVworp:plot} reports estimates of the rank-frequency distribution obtained with $1$-pass and $2$-pass WORp and perfect WOR ($p$-ppswor) and perfect WR samples (shown for reference).  For best comparison, all WOR methods use the same randomization of the $p$-ppswor transform. 
Table~\ref{tab:errors} reports normalized root averaged squared errors (NRMSE) on example statistics. 
As expected, 2-pass WORp and perfect $2$-ppswor are similar and WR $\ell_2$ samples are less accurate with larger skew or on the tail. Note that current state of the art sketching methods are not more efficient for WR sampling than for WOR sampling, and estimation quality with perfect methods is only shown for reference.  We can also see that 1-pass WORp performs well, although it requires more accuracy (lager sketch size) since it works with estimated weights (reported results are with fixed \texttt{CountSketch} size of $k\times 31$ for all methods).

\begin{figure}[ht]
\centering
\centerline{
\includegraphics[width=0.3\textwidth]{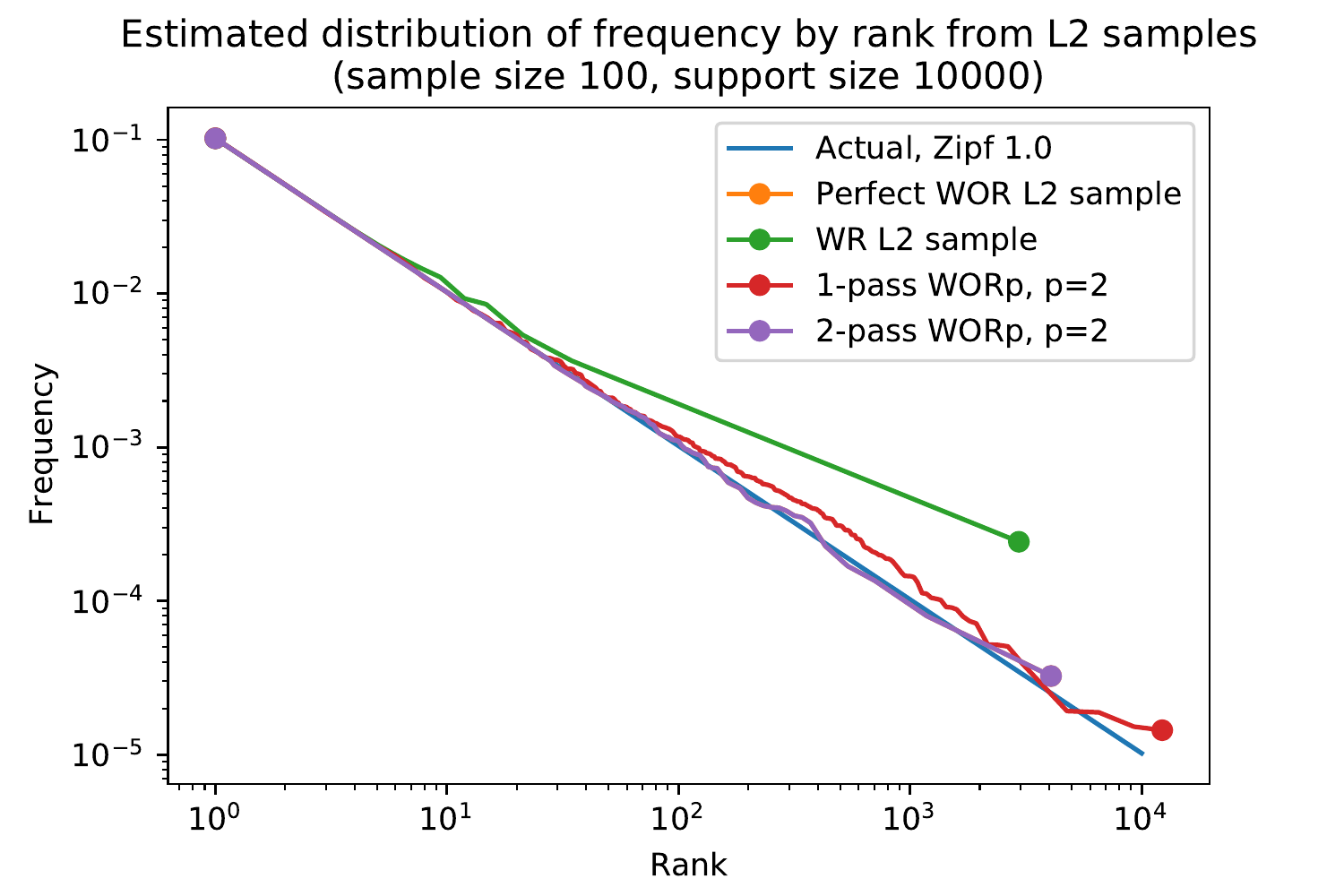}
\includegraphics[width=0.3\textwidth]{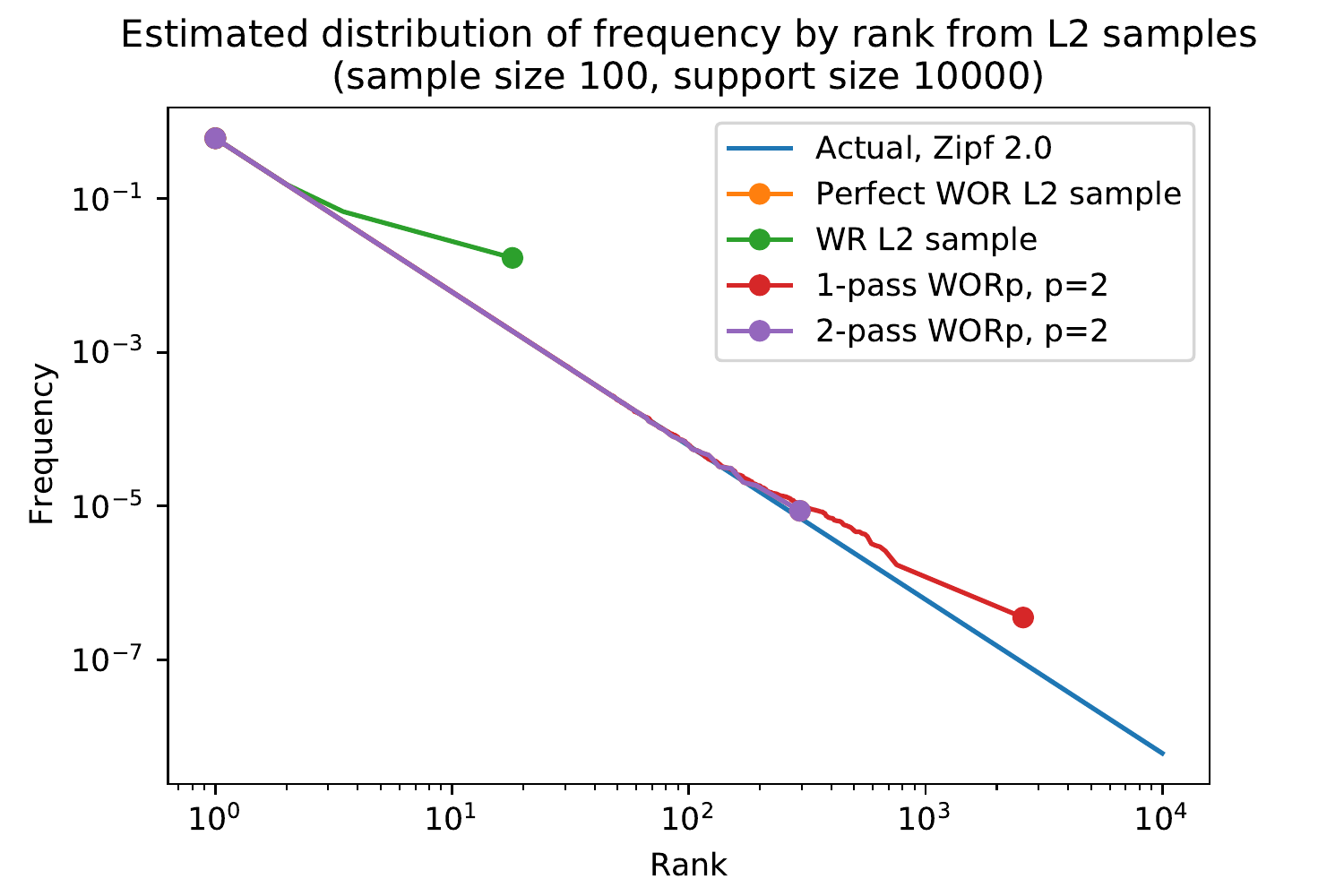}
\includegraphics[width=0.3\textwidth]{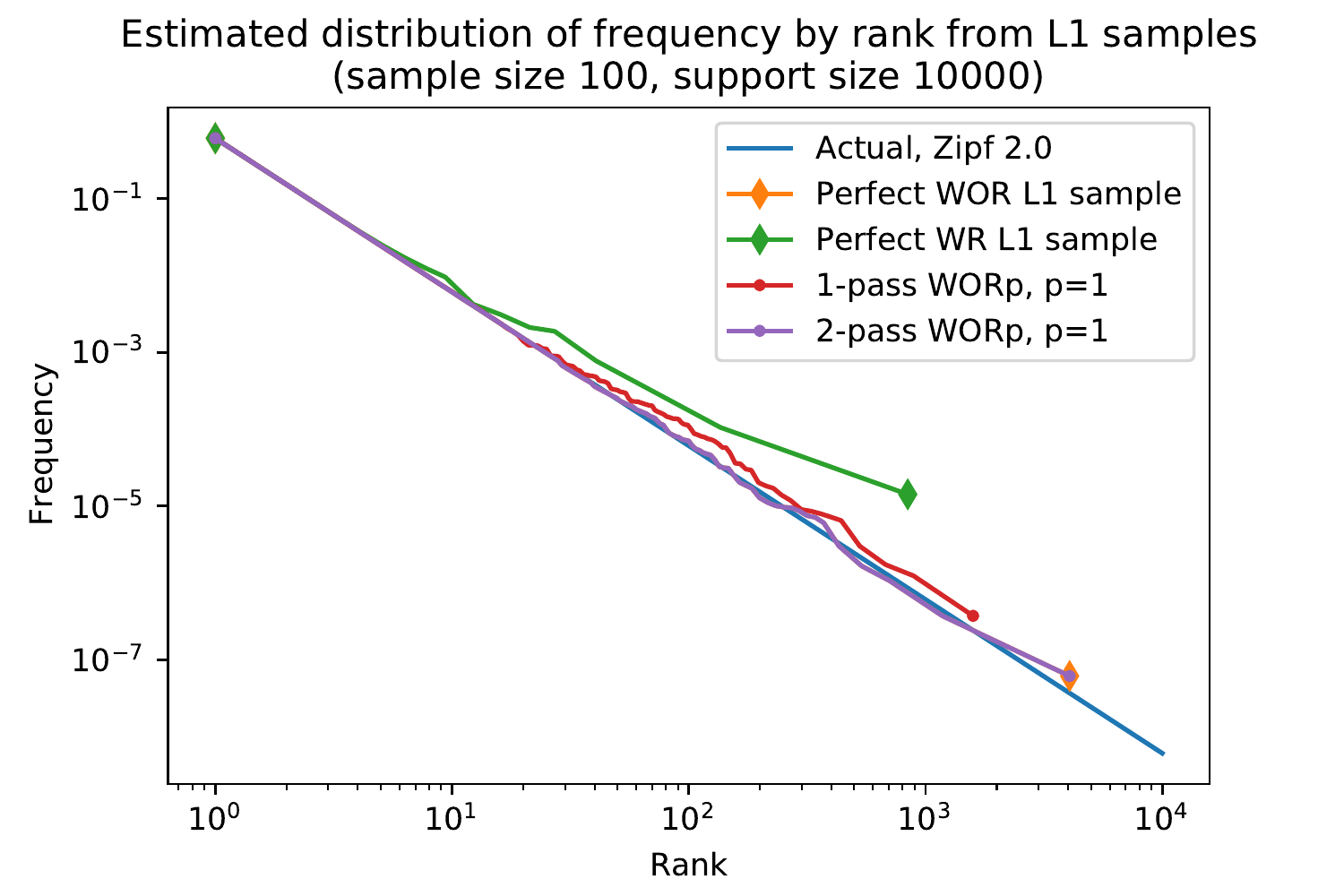}
}
\caption{Estimates of the rank-frequency distribution of $\Zipf[1]$ and $\Zipf[2]$.  Using WORp 1-pass, WORp 2-pass with \texttt{CountSketch} (matrix $k\times 31$), perfect WOR, and perfect WR.
Estimates from a (representative) single sample of size $k=100$. Left and Center: $\ell_2$ sampling. Right: $\ell_1$ sampling.
}
\label{perfectVworp:plot}
\end{figure}

\begin{table}[]
    \centering
    \small
    \begin{tabular}{r rrr||c | c c }
   $\ell_p$ &  $\alpha$ & $p'$ & perfect WR & perfect WOR & 1-pass WORp & 2-pass WORp \\
\hline        
$\ell_2$ & $\Zipf[2]$ & $\nu^3$ 
 & 1.16e-04 & 2.09e-11 & 1.06e-03 & 2.08e-11\\
$\ell_2$ & $\Zipf[2]$ & $\nu^2$  
& 7.96e-05 & 1.26e-07 & 1.14e-02 & 1.25e-07\\
\hline
$\ell_1$ & $\Zipf[2]$ & $\nu$ 
& 9.51e-03 & 1.60e-03 & 2.79e-02 & 1.60e-03 \\

$\ell_1$ & $\Zipf[1]$ & $\nu^3$ 
& 3.59e-01 & 5.73e-03 & 5.14e-03 & 5.72e-03 \\
$\ell_1$ & $\Zipf[2]$ & $\nu^3$ 
& 3.45e-04 & 7.34e-10 & 5.11e-05 & 7.38e-10\\
    \end{tabular}
    \caption{NRMSE on estimates of frequency moments on statistics of the form $\|\bnu\|_{p'}^{p'}$ from $\ell_p$ samples ($p=1,2$). $\Zipf[\alpha]$ distributions with support size $n=10^4$, $k=100$ samples, averaged over $100$ runs. \texttt{CountSketch} of size $k \times 31$}
    \label{tab:errors}
    \vspace{-0.3in}
\end{table}

\subsection*{Conclusion}
We present novel composable sketches for without-replacement (WOR)  $\ell_p$ sampling, based on
``reducing" the sampling problem to a heavy hitters (HH) problem. The reduction, which is simple and practical,  allows us to use existing implementations of HH sketches to perform WOR sampling.  Moreover, streaming HH sketches that support time decay (for example, sliding windows~\cite{Ben-BasatWFK:Infocom2016}) provide  a respective time-decay variant of sampling.  We present two approaches, WORp, based on a bottom-$k$ transform, and another technique based on ``perfect'' with-replacement sampling sketches, which provide 1-pass WOR samples with
negligible variation distance to a true sample.  Our methods open the door for a wide range of future applications.  
Our WORp schemes produce bottom-$k$ samples (aka bottom-$k$ sketches) with respect to a specified randomization $r_x$ over the support
(with 1-pass WORp we obtain approximate bottom-$k$ samples).  Samples of different datasets or different $p$ values or different time-decay functions that are generated with the same $r_x$ are {\em coordinated}~\cite{BrEaJo:1972,Saavedra:1995,Ohlsson:2000,ECohen6f,CoWaSu, BRODER:sequences97}.  Coordination is a desirable and powerful property:
Samples are locally sensitivity (LSH) and change little with small changes in the dataset~\cite{BrEaJo:1972,Saavedra:1995,IndykMotwani:stoc98,CohenCDL:JAlg2016,multiobjective:2015}. This LSH property allows for a compact representation of multiple samples,  efficient updating, and sketch-based similarity searches. Moreover, coordinated samples (sketches) facilitate powerful estimators for multi-set statistics and similarity measures such as weighted Jaccard similarity, min or max sums, and one-sided distance norms~\cite{ECohen6f,BRODER:sequences97,CK:sigmetrics09,multiw:VLDB2009,CK:pods11,CKsharedseed:2012,sdiff:KDD2014}.

\paragraph{Acknowledgments:} D. Woodruff would like to thank partial support from the Office of Naval Research (ONR) grant N00014-18-1-2562, and the National Science Foundation (NSF) under Grant No. CCF-1815840.   

\onlyinproc{
\newpage
\section{Broader Impact}
Broader Impact Discussion is not applicable.  We presented a method for WOR sampling that has broad applications in ML.  But this is a technical paper with no particular societal implications.  
}
\bibliographystyle{plain}  
\bibliography{cycle}

\appendix

 \section{Properties of rHH sketches} \label{rHHmore:sec}
 
\paragraph{Sketches for $\ell_1$ heavy hitters on datasets with
  positive values:}  These include the deterministic counter-based Misra Gries~\cite{MisraGries:1982,DBLP:journals/tods/AgarwalCHPWY13}, Lossy Counting~\cite{MM:vldb2002}, and Space Saving~\cite{spacesaving:ICDT2005} and the randomized Count-Min Sketch~\cite{CormodeMuthu:2005}.  A sketch of size $O(\varepsilon^{-1})$ provides frequency estimates with absolute error at most $\varepsilon \|\boldsymbol{\nu}\|_1$. 
Berinde et al. \cite{BerindeCIS:PODS2009} provide a counter-based sketch of size
$O(k/\psi)$ that provides absolute error at most $\frac{\psi}{k}\|\tail_{k}(\boldsymbol{\nu})\|_1$.

\paragraph{Sketches for $\ell_2$ heavy hitters on datasets with signed values:}
Pioneered by \texttt{CountSketch}~\cite{ccf:icalp2002}:
A sketch of size $O(\varepsilon^{-1}\log \frac{n}{\delta})$  provides
with confidence $1-\delta$ estimates with error bound $\varepsilon
\|\boldsymbol{\nu}\|_2^2$ for squared frequencies.
For rHH, a \texttt{CountSketch} of size $O(\frac{k}{\psi} \log \frac{n}{\delta})$ provides estimates
for all squared frequencies with absolute error at most
$\frac{\psi}{k} \|\tail_{k}(\boldsymbol{\nu})\|_2^2$.  These bounds were
further refined in
\cite{Jowhari:Saglam:Tardos:11}  for $\ell_p$
rHH.  The dependence on $\log n$ was replaced by $1/\psi$ in
\cite{BravermanCINWW:PODS2017} for insertion only streams. Unlike the case for counter-based sketches, the estimates produced by \texttt{CountSketch} are unbiased, a helpful property for estimation tasks. 

\paragraph*{Obtaining the rHH keys}
Keys can be arbitrary strings (search queries, URLs, terms) or
integers from a domain $[n]$ (parameters in an ML model).  
Keys in the form of strings can be handled by hashing them to a domain $[n]$ but we discuss applications that require the sketch to return rHH keys in their string form.
Counter-based sketches store explicitly
$O(k/\psi)$ keys.  The stored keys can be arbitrary strings.  The estimates are positive for stored keys and $0$ for other keys. The rHH keys are contained in the stored keys.
The randomized rHH sketches (\texttt{CountSketch} and \texttt{CountMin}) are designed for keys that are integers in $[n]$.  The bare sketch does not explicitly store keys.  The rHH keys can be recovered by enumerating over $[n]$ and retaining the keys with largest estimates.  Alternatively, when streaming (performing one sequential pass over elements) we can maintain an auxiliary structure that holds key strings with current top-$k$ estimates~\cite{ccf:icalp2002}. 

With general composable sketches, key strings can be handled using an auxiliary structure  that increases the sketch size by a factor linear in string length.  This is inefficient compared with sketches that natively store the string.  Alternatively, a two-pass method can be used with the first pass computing an rHH sketch for a hashed numeric representation and a second pass is used to obtain the key strings of hashed representations with high estimates.

\ignore{
****** revise below ********
and respective estimated frequencies.  Moreover, in applications the keys are often strings from a large domain (search queries, URLs, terms).  They can be hashed to a numeric domain of size $O(n)$ but eventually we still need to return the string representation.
Counter-based sketches maintain .  Estimates are $0$ for keys not stored in the sketch and positive otherwise and the stored keys include the rHH keys. 
Random-projection sketches (Count-sketch and Count-Min) are designed for keys that are integers $[n]$. The bare sketch provides approximate frequencies for numerically represented query keys.  But the sketch does not explicitly track keys. When streaming, we can use
an auxliary structure that tracks and stores key strings with numeric hashes with largest estimated values. In a general distributed aggregation setting this method does not apply.  One solution is to use a second pass to collects key strings with highest estimates and their exact frequencies. Another solution is to enumerate over the numeric domain and identify (numeric) keys with highest estimates (but this does not allow for recovery of key strings).   A third approach is to use an auxiliary structures that uses estimates of the bit representation of the keys. This multiplies sketch size by the key string length.
}

\paragraph*{Recovering (approximate) frequencies of rHH keys}
For our application, we would need to have approximate or exact frequencies of rHH keys.  
 The estimates provided by a $(k,\psi)$ rHH sketch provide absolute error (statistical) guarantees (see Table~\ref{rHH:table}).  One approach is to recover exact frequencies in a second pass.  We can also obtain more accurate estimates (of relative error at most $\epsilon$) by using a $(k,\epsilon\psi)$ rHH sketch.

\paragraph*{Testing for failure}
Recall that the dataset may not have $(k,\psi)$ rHH.   We can test
if one of the $k$ largest estimated frequencies to the $p$th power  is below the specified error
   bound of $\geq \frac{\psi}{k}
   \|\tail_{k}(\boldsymbol{\nu})\|_p^p$.  If so,  we declare    ``failure.''

\section{Overview of the proof of Theorem~\ref{rHHprop:thm}} \label{Oproof:sec}
For a vector $\boldsymbol{w}\in\Re^n$, permutation $\pi\in S^n$,  and $p>0$, 
let the random variable $\boldsymbol{w^*}\sim\text{$p$-ppswor}[\boldsymbol{w}] \mid\pi$ be a $p$-ppswor transform \eqref{bottomkptransform:eq}  of $\boldsymbol{w}$ conditioned on the event $\order(\boldsymbol{w^*})=\pi$.
  For $q>p$ and $k>1$, we define the following distribution:
 \begin{equation} \label{ratio:eq}
 F_{\boldsymbol{w},p,q,k}\mid\pi := \frac{\|\tail_{k}(\boldsymbol{w^*})\|_q^q}{(w^*_{(k)})^q}\ .
 \end{equation}
Note that for any $\boldsymbol{w}\in \Re^n$ and $\pi\in S^n$,

\begin{equation}\label{F2psi:eq}
\Pr_{\boldsymbol{w}^*\sim\text{$p$-ppswor}[\boldsymbol{w}] \mid \order(\boldsymbol{w}^*) =\pi} \left[
          k \frac{|w^*_{(k)}|^q}{\|\tail_{k}(\boldsymbol{w}^*)\|_q^q} \leq  \psi\right] =
          \Pr_{z\sim F_{\boldsymbol{w},p,q,k}\mid\pi }\left[z \leq \frac{k}{\psi}\right] 
\end{equation}
Therefore tail bounds on $F$ that hold for any
 $\boldsymbol{w}\in \Re^n$ and $\pi\in S^n$ can be used to establish the claim.

We now define another distribution that does not depend on $\boldsymbol{w}$ and $\pi$:
\begin{definition} \label{Rdist:def}
For $1\leq k\leq n$ and $\rho\geq 1$ we define a distribution $R_{n,k,\rho}$ as follows. 
\[
R_{k,n,\rho} := \sum_{i=k+1}^n \frac{\left(\sum_{j=1}^{k} Z_j\right)^\rho}{\left(\sum_{j=1}^{i} Z_j\right)^\rho }\enspace ,
\]
where $Z_i\sim \Exp[1]$ $i\in [n]$ are i.i.d.
\end{definition}

 The proof of Theorem~\ref{rHHprop:thm} will follow using the following two components:
 \begin{enumerate}
     \item [(i)]
 We show (Section \ref{dom:sec}) that for any $\boldsymbol{w} \in \Re^n$ and permutation $\pi\in S^n$, 
 \[
F_{\boldsymbol{w},p,q,k}|\pi \preceq R_{k,n,\rho=(q/p)}\ ,
\]
where the relation $\preceq$ corresponds here to statistical domination of distributions. 
\item [(ii)]
 We establish (Section~\ref{proofRbound:sec}) tail bounds on $R_{k,n,\rho=(q/p)}$.
\end{enumerate}
Because of domination, the tails bounds on $R_{k,n,\rho=(q/p)}$ provide corresponding tail bound for $F_{\boldsymbol{w},p,q,k}|\pi$ for any $\boldsymbol{w}\in\Re^n$ and $\pi\in S^n$.  Together with \eqref{F2psi:eq}, we use the tail bounds to conclude the proof of Theorem~\ref{rHHprop:thm}.  

Moreover, the domination relation is tight in the sense that for some
$\boldsymbol{w}$ and $\pi$, $F_{\boldsymbol{w},p,q,k}|\pi$ is very close to 
$R_{k,n,q/p}$:
For distributions with $k$ keys with relative frequency $\epsilon$ and $\pi$ that has these keys in the first $k$ (as $\epsilon\rightarrow 0$),  or for uniform distributions with $n\gg k$, $F_{\boldsymbol{w},p,q,k}|\pi$ (as $n$ grows). 

The tail bounds (and hence the claim of Theorem~\ref{rHHprop:thm}) also hold without the condition on $\pi$.
\begin{lemma}
The tail bounds also hold for the unconditional distribution $F_{\boldsymbol{w},p,q,k}$.
\end{lemma}
\begin{proof}
The distribution $F_{\boldsymbol{w},p,q,k}$ is a convex combination of distributions $F_{\boldsymbol{w},p,q,k}|\pi$. Specifically, for each permutation $\pi$ let $p_{\pi}$ be the probability that we obtain this permutation with successive weighted sampling with replacement.  
Then
\begin{equation}\label{convexcombo:eq}
    F_{\boldsymbol{w},p,q,k} = \sum_{\pi} p_{\pi} F_{\boldsymbol{w},p,q,k}|\pi \enspace .
\end{equation}
Since tail bounds hold for each term, they hold for the combination.
\end{proof}

\subsection{Approximating \texorpdfstring{$\boldsymbol{\psi}$}{Psi} by simulations} \label{psiapprox:sec}
$\Psi_{k,n,\rho}(\delta)$ 
is the solution of the following for $\psi$:
\begin{equation}\label{psitouse:eq}
\Pr_{z \sim R_{k,n,\rho}}[z \geq k/\psi]=\delta\ .
\end{equation}   
We can approximate $\Psi_{k,n,\rho}(\delta)$ by computing i.i.d.\ $z_i \sim  R_{k,n,\rho}$, taking the $(1-\delta)$ quantile $z'$ in the empirical distribution and returning $k/z'$.

From simulations we obtain that for $\delta = 0.01$ and $\rho\in\{1,2\}$,
 $C= 2$ suffices for sample size $k\geq 10$,  $C = 1.4$ suffices for $k\geq 100$, and $C=1.1$ suffices for $k\geq 1000$.  

\section{Domination of the ratio distribution} \label{dom:sec}
  

\begin{lemma}[Domination]  \label{Rdomination:lemma}
For any permutation $\pi$, $\boldsymbol{w}$, $p$, $q\geq p$, and $k\geq 1$, the distribution $F_{\boldsymbol{w},p,q,k}|_\pi$ \eqref{ratio:eq}
 is dominated by $R_{n,k,q/p}$.   That is, for all $z\geq 0$,
 \begin{equation}
     \Pr_{z\sim F_{\boldsymbol{w},p,q,k}|\pi}\left[z \leq \frac{k}{\psi}\right] \geq
               \Pr_{z\sim R_{n,k,q/p}}\left[z \leq \frac{k}{\psi}\right]
 \end{equation}
 \end{lemma}  
\begin{proof}  
Assume without loss of generality that $\order(\boldsymbol{w})=\pi$. Let
$\boldsymbol{w}^* \sim \text{$p$-ppswor}[\boldsymbol{w}]\mid\order(\boldsymbol{w}^*)=\pi$. Note
by definition $\boldsymbol{w}^*$  is in decreasing order of magnitude.
Define the random variable $\boldsymbol{y}  := \boldsymbol{w^*}^p$.
$\boldsymbol{y}$ are transformed weights of a ppswor sample of $\boldsymbol{w}$ conditioned on the order $\pi$.  We use use properties of the exponential distribution (see a similar use in~\cite{bottomk07:est}) to express the joint distribution of $\{y_i\}$.  We use the following set of 
independent random variables: 
\[
X_i \sim \Exp[\sum_{j=i}^{n} w^p_j]\enspace .
\]
We have:
\begin{equation} \label{sumdist:eq}
  y_i = \frac{1}{\left( \sum_{j=1}^{i} X_i\right)^{q/p}} \enspace . 
\end{equation}
To see this, recall that $y_1$ is the (inverse) of the minimum of 
exponential random variables with parameters $w_1,\ldots,w_n$ and thus is (the inverse of) exponential random variable with parameter equal to their sum.  Therefore, $y_1 = 1/X_1$. From memorylessness, the difference between the $(i+1)$-st smallest inverse and the $i$-th smallest is an exponential random variable with distribution $X_i$.  Therefore, the $i$-th smallest inverse has the claimed distribution~\eqref{sumdist:eq}.

We are now ready to express the random variable that is the ratio \eqref{ratio:eq} in terms of the independent random variables $X_i$:
\begin{equation}\label{ratioexpress:eq}
\frac{\sum_{j=k+1}^n y_j }{y_k}= \frac{\sum_{i=k+1}^n \frac{1}{\left(\sum_{j=1}^{i} X_i\right)^{q/p} }}{\frac{1}{\left(\sum_{j=1}^{k} X_j\right)^{q/p}}} =
\sum_{i=k+1}^n \frac{\left(\sum_{j=1}^{k} X_j\right)^{q/p}}{\left(\sum_{j=1}^{i} X_j\right)^{q/p} }\enspace .
\end{equation}

We rewrite this using i.i.d.\ random variables $Z_i \sim \Exp[1]$, recalling that
for any $w$, $\Exp[w]$ is the same as $\Exp[1]/w$.
Then we have $X_i = Z_i/\sum_{j=i}^{n} w^p_j$.

We next provide a simpler distribution that dominates the distribution of the ratio.
Let $W' :=  \sum_{j=k}^{n} w^p_j$ and consider the i.i.d.\ random variables $X'_i = Z_i/W'$.  
Note that $X_j \leq X'_j$ for $j\leq k$ and $X_j \geq X'_j$ for $j\geq k$.
Thus, for $i\geq k+1$,
\begin{equation}
    \frac{\sum_{j=1}^{k} X_j}{\sum_{j=1}^{i} X_j} = \frac{1}{1+\frac{\sum_{j=k+1}^{i} X_j}{\sum_{j=1}^{k} X_j}}\geq  \frac{1}{1+\frac{\sum_{j=k+1}^{i} X'_j}{\sum_{j=1}^{k} X'_j}} = \frac{\sum_{j=1}^{k} X'_j}{\sum_{j=1}^{i} X'_j}=\frac{\sum_{j=1}^{k} Z_j}{\sum_{j=1}^{i} Z_j}
\end{equation}
This holds in particular for each term in the RHS of \eqref{ratioexpress:eq}.
Therefore we obtain 
\[
\frac{\sum_{j=k+1}^n y_j }{y_k}\geq \sum_{i=k+1}^n \frac{\left(\sum_{j=1}^{k} Z_j\right)^{q/p}}{\left(\sum_{j=1}^{i} Z_j\right)^{q/p}}\enspace .
\]
\end{proof}



\section{Tail bounds on $R_{k,n,\rho}$} \label{proofRbound:sec}

We establish the following upper tail bounds on the distribution $R_{n,k,\rho}$:
 \begin{theorem} [Concentration of $R_{n,k,\rho}$] \label{Ruppertail:thm}
There is a constant $C$, such that for any $n,k,\rho$
\begin{align}
\rho=1 \text{:} &
 \Pr_{r\sim R_{n,k,\rho}}\left[r \geq
                  C k \ln(\frac{n}{k}) \right] \leq  3e^{-k}
  \\
  \rho>1 \text{:} &
 \Pr_{r\sim R_{n,k,\rho}}\left[r \geq C k \frac{1}{\rho-1}\right]  \leq 3 e^{-k}
\end{align}
\end{theorem}

 We start with a ``back of the envelope'' calculation to provide intuition: replace the random variables $Z_i$ in $R_{n,k,\rho}$ 
(see Definition~\ref{Rdist:def}) by their expectation $\E[Z_i]=1$ to obtain 
\[S_{n,k,\rho}:= \sum_{i=k+1}^n \frac{k^\rho}{i^\rho} \enspace . \]
For $\rho=1$, $S_{n,k,\rho}\leq k(H_n-H_k) \approx k \ln(n/k)$.  For $\rho>1$ we have $S_{n,k,\rho}\approx \frac{k}{\rho-1}$.  
We will see that we can expect the sums not to deviate too far from this value.

The sum of $\ell$ i.i.d.\ $\Exp[1]$ random variables generates an Erlang distribution $\Erlang[\ell,1]$ (rate parameter $1$).
The expectation is $\E_{r\sim \Erlang[\ell,1]} = \ell$  and variance is $\Var_{r\sim \Erlang[\ell,1]}[r] = \ell$.
We will use the following Erlang tail bounds~\cite{ErlangTails:2017}:
  \begin{lemma} \label{Erlangtails:lemma}
For $X\sim \Erlang[\ell,1]$
\begin{eqnarray*}
\varepsilon \geq 1 \text{ : }& \Pr[x \geq \varepsilon\ell] \leq \frac{1}{\varepsilon} e^{-\ell(\varepsilon-1-\ln \varepsilon)}\leq e^{1-\varepsilon} \\
\varepsilon \leq 1 \text{ : } & \Pr[x \leq \varepsilon\ell] \leq e^{-\ell(\varepsilon-1-\ln\varepsilon)} 
\end{eqnarray*} 
\end{lemma} 
When $\varepsilon < 0.159$ or $\varepsilon > 3.2$ we have the bound
$e^{-\ell}$.  For $\varepsilon> 3.2$  we also have
$\frac{1}{\varepsilon} e^{-\ell(\varepsilon-2.2)}$

\begin{proof} [Proof of Theorem~\ref{Ruppertail:thm}]

We bound the probability of 
a ``bad event"  which we define as the numerator being ``too
large'' and  denominators being too ``small.''
More formally,
the numerator is the sum $N=\sum_{i=1}^k Z_i$ and we
define a bad event as   $N \geq 3.2 k$.
  Substituting $\varepsilon = 3.2$ and 
  $\ell=k$  in the
  upper tail bounds from  Lemma~\ref{Erlangtails:lemma}, we
  have that the probability of this bad event is bounded by
\begin{equation}  \label{badnum:eq}
 \Pr_{r\sim \Erlang[k,1]}\left[ r > k\varepsilon  \right] \leq 
 e^{-k}\enspace .
\end{equation} 

The denominators are prefix sums of of the sequence of random
variables.
We consider a partition the sequence $Z_{k+1},\ldots, Z_n$ to consecutive 
stretches of size 
\[ \ell_h := 2^h k ,   (h\geq 1) \ . \]
We denote by $S_h$ the sum of stretch $h$. 
Note that $S_h\sim \Erlang[\ell_h,1]$ are independent random variables.
We define a bad event as the probability that for some $h\geq 1$, 
$S_h \leq 0.15 \ell_h = 0.14\ 2^h k$.
 From the lower tail bound of  Lemma~\ref{Erlangtails:lemma}, we
  have
\begin{equation}\label{baddenom:eq}
  \Pr[S_h\leq 0.15\ell_h] = \Pr_{r\sim \Erlang[\ell_h,1]}\left[ r <  0.15 \ell_h  \right]  \leq
 e^{-\ell_h} \leq  e^{-2^h k}\enspace .
\end{equation}  
The combined probability of the union of these bad events (for the
numerator and all stretches) is at most
$e^{-k} + \sum_{h\geq 1} e^{-2^h k} \leq 3e^{-k}$.

We are now ready to compute probabilistic upper bound on the ratios
when there is no bad event
\begin{eqnarray*}
  R_{n,k,\rho} &\leq& \sum_{h\geq 1} \ell_h \frac{N^\rho }{(N+ \sum_{i<h}
                      S_i)^\rho } \\
  &\leq& \sum_{h\geq 1} \ell_h \frac{\left(3.2 k\right)^\rho}{\left(3.2 k + 0.15 \sum_{i<h} \ell_i  \right)^\rho}\\
 &=&  2k \frac{ \left(3.2 k \right)^\rho }{\left( 3.2 k \right)^\rho
     } +   \sum_{h\geq 2} 2^h k \frac{\left(3.2 k \right)^\rho}{\left(3.2 k + (2^{h}-2) k \right)^\rho}\\
  &\leq&  k (2+ \sum_{h= 2}^{\lceil \log_2(n/k)\rceil} 2^h
         \left(\frac{3.2}{2^h+1.2}\right)^\rho \leq k \left(2+  3.2^\rho \sum_{h= 2}^{\lceil \log_2(n/k)\rceil}  2^{-h(\rho-1)}\right)
  \end{eqnarray*}
For $\rho=1$ we have $O(k \log n)$.  For
$\rho>1$, we have $O(k/(\rho-1))$.   
\end{proof}   

From the proof of Theorem~\ref{Ruppertail:thm} we obtain:
\begin{corollary} \label{constB:coro}
There is a constant $B$ such that
when there are no ``bad events" in the sense of the proof of Theorem~\ref{Ruppertail:thm}, 
\[
\frac{\sum_{i=1}^{k} Z_i}{\sum_{i=1}^{B k} Z_i} \leq 1/3\ .
\]
\end{corollary}   
\begin{proof}
 With no bad events, $N = \sum_{i=1}^{k} Z_i < 3.2k$ and
 $\sum_{i=k+1}^{k 2^h } Z_i \geq 0.15 k (2^h-1)$.
 Solving for $0.15 k B \geq 6.4 k$ (for $B =2^h -1$ for some $h$) we obtain $B=63$.
 \end{proof}

 
 \section{Ratio of magnitudes of transformed weights} \label{keyration:sec}

  For $k_2>k_1$ we consider the distribution of the ratio between the $k_2^{th}$ and $k_1^{th}$ transformed weights:
   \[
   G_{\boldsymbol{w},p,q,k_1,k_2} \mid \pi := 
   \left| \frac{w^{*p}_{(k_2)}}{w^{*p}_{(k_1)}}\right|
   \]
  
  \begin{lemma} \label{domB:lemma}
  For any $\boldsymbol{w}\in \Re^n$, $\pi\in S^n$, and $k_1< k_2 \leq n$, the distribution  $G_{\boldsymbol{w},p,q,k_1,k_2} \mid \pi$ is dominated by
  \begin{equation}
       G'_{\rho=q/p,k_1,k_2} := \left(\frac{\sum_{i=1}^{k_1} Z_i}{\sum_{i=1}^{k_2} Z_i}\right)^\rho\enspace ,
    \end{equation}
    where $Z_i\sim\Exp[1]$ are i.i.d.
  \end{lemma}
  \begin{proof}
   Following the notation in the proof of Lemma~\ref{Rdomination:lemma},
   the distribution $G_{\boldsymbol{w},p,q,k_1,k_2}$ can be expressed as
   \[
   \left(\frac{\sum_{i=1}^{k_1} X_i}{\sum_{i=1}^{k_2} X_i}\right)^\rho
   \]
   where $X_i := \frac{Z_i}{\sum_{j=i}^n w_j^p}$.
   
   For $i\in [n]$ we define $X'_i := \frac{Z_i}{\sum_{j=k_1}^n w_j^p}$.
   Now note that $X'_i \geq X_i$ for $i\leq k_1$ and $X'_i \geq X_i$ for $i\geq k_1$.
   Therefore,
   \begin{align*}
      \frac{\sum_{i=1}^{k_1} X_i}{\sum_{i=1}^{k_2} X_i} &= \frac{1}{1+\frac{\sum_{i=k_1+1}^{k_2} X_i}{\sum_{i=1}^{k_1} X_i}} \\
      &\leq \frac{1}{1+\frac{\sum_{i=k_1+1}^{k_2} X'_i}{\sum_{i=1}^{k_1} X'_i}} =
   \frac{1}{1+\frac{\sum_{i=k_1+1}^{k_2} Z_i}{\sum_{i=1}^{k_1} Z_i}} = 
   \frac{\sum_{i=1}^{k_1} Z_i}{\sum_{i=1}^{k_2} Z_i}\enspace .
   \end{align*}
  \end{proof}

\section{1-pass with total variation distance on sample $k$-tuple: upper and lower bounds}\label{onepassTD:sec}

Perfect ppswor returns each subset of $k$ keys $S=\{i_1,\ldots,i_k\}$ with a certain probability:
\[
p(S) = \sum_{\pi \mid \{\pi_1,\ldots,\pi_k\}=S} \prod_{j=1}^k \frac{w_{i_j}}{\|\boldsymbol{w}\|_1 -\sum_{h<j}w_{i_h}}\ .
\]
Recall that the distribution is equivalent to successive weighted sampling without replacement.  It is also equivalent to successive sampling with replacement if we ``skip" repetitions until we obtain $k$ distinct keys.

With $p$-ppswor and unaggregated data, this is with respect to $\nu_x^p$.  The WORp 1-pass method returns an approximate $p$-ppswor sample in terms of estimation quality and per-key inclusion probability  but the TV distance on $k$-tuples can be large.

We present here another 1-pass method that returns a $k$-tuple with a polynomially small VT distance from $p$-ppswor.

\begin{algorithm2e}\caption{1-pass Low Variation Distance Sampling}\label{lowTVD:alg}
 \KwIn{$\ell_p$ rHH method, perfect $\ell_p$-single sampler method, sample size $k$, $p$, $\delta$, $n$,}
 \DontPrintSemicolon
 {\bf Initialization:}\;
\hspace{5mm} Initialize $r = C \cdot k \log n$ independent perfect $\ell_p$-single sampling algorithms $A^1, \ldots, A^r$.\; 
\hspace{5mm} Initialize an $\ell_p$ rHH method $R$.\;
 {\bf Pass 1:}\;
\hspace{5mm} Feed each stream update into $A^1, \ldots, A^r$ as well as into $R$.\;
  {\bf Produce sample:}\;
\hspace{5mm} $S \leftarrow \emptyset$\;
\hspace{5mm} For $i = 1, \ldots, r$\;
\hspace{7mm} Let $Out_i$ be the index returned by $A^i$\;
\hspace{7mm} If $Out_i \notin S$, then\;
\hspace{9mm} $S \leftarrow S \cup \{Out_i\}$\;
\hspace{9mm} For each $j > i$, feed the update $x_{Out_i} \leftarrow x_{Out_i} - R(Out_i)$ into $A^j$ \tcp*{$R(Out_i)$ is the estimate of $x_i$ given by $R$}\;
\hspace{7mm} If $|S| = k$ then exit and return $S$\;
\hspace{5mm} end\;
Output {\sc Fail}\tcp*{Algorithms returns $S$ before reaching this line with high probability}
 \end{algorithm2e}

\begin{theorem}\label{thm:tvd}
Let $p \in (0,2]$. There is a $1$-pass turnstile streaming algorithm using $k \cdot \textrm{poly}(\log n)$ bits of memory which, given a stream of updates to an underlying vector $x \in \{-M, -M+1, \ldots, M-1, M\}^n$, with $M = \textrm{poly}(n)$, outputs a set $S$ of size $k$ such that the distribution of $S$ has variation distance at most $\frac{1}{n^C}$ from the distribution of a sample without replacement of size $k$ from the distribution $\mu = (\mu_1, \ldots, \mu_n)$, where $\mu_i = \frac{|x_i|^p}{\|x\|_p^p}$, where $C > 0$ is an arbitrarily large constant.  
\end{theorem}
\begin{proof}
The algorithm is $1$-pass and works in a turnstile stream given an $\ell_p$ rHH method and perfect $\ell_p$-single sampler methods that have this property. We use the $\ell_p$ rHH method of \cite{Jowhari:Saglam:Tardos:11}, which has this property and uses $O(k \cdot \log^2 n)$ bits of memory. We also use the perfect $\ell_p$-single sampler method of \cite{JayaramW:Focs2018}, which has this property and uses $\log^2 n \cdot \textrm{poly}(\log \log n)$ bits of memory for $0 < p < 2$ and $O(\log^3 n)$ bits of memory for $p = 2$. The perfect $\ell_p$-single sampler method of \cite{JayaramW:Focs2018} can output {\sc Fail} with constant probability, but we can repeat it $C \log n$ times and output the first sample found, so that it outputs {\sc Fail} with probability at most $\frac{1}{n^C}$ for an arbitrarily large constant $C > 0$, and consequently we can assume it never outputs {\sc Fail} (by say, always outputting index $1$ when {\sc Fail} occurs). This gives us the claimed $k \cdot \textrm{poly}(\log n)$ total bits of memory. 

We next state properties of these subroutines. The $\ell_p$ rHH method we use satisfies: with probability $1-\frac{1}{n^C}$ for an arbitrarily large constant $C > 0$, simultaneously for all $j \in [n]$, it outputs an estimate $R(j)$ for which 
$$R(j) = x_i \pm \left (\frac{1}{2k} \right ) ^{1/p} \cdot \|\tail_k(x)\|_p.$$
We assume this event occurs and add $\frac{1}{n^C}$ to our overall variation distance. 

The next property concerns the perfect $\ell_p$-single samplers $A^j$ we use. Each $A^j$ returns an index $i \in \{1, 2, \ldots, n\}$ such that the distribution of $i$ has variation distance at most $\frac{1}{n^C}$ from the distribution $\mu$. Here $C > 0$ is an arbitrarily large constant of our choice.

We next analyze our procedure for producing a sample. Consider the joint distribution of $(Out_1, Out_2, \ldots, Out_{2Ck \log n})$. The algorithms $A^i$ use independent randomness. However, the input to $A^i$ may depend on the randomness of $A^{i'}$ for $i' < i$. However, by definition, conditioned on $A^i$ not outputting $Out_{i'}$ for any $i' < i,$ we have that $Out_i$ is independent of $Out_{1}, \ldots, Out_{i-1}$ and moreover, the distribution of $Out_i$ has variation distance $\frac{1}{n^C}$ from the distribution of a sample $s$ from $\mu$ conditioned on $s \notin \{Out_1, \ldots, Out_{i-1}\}$, for an arbitrarily large constant $C> 0$.  

Let $E$ be the event that we sample $k$ distinct indices, i.e., do not output {\sc Fail} in our overall algorithm. We show below that $\Pr[E] \geq 1-\frac{1}{n^C}$ for an arbitrarily large constant $C> 0$. Consequently, our output distribution has variation distance $\textrm{1}{n^C}$ from an idealized algorithm that samples until it has $k$ distinct values. 

Consider the probability of outputting a particular ordered tuple $(i_1, \ldots, i_k)$ of $k$ distinct indices in the idealized algorithm that samples until it has $k$ distinct values. By the above, this is 
$$\prod_{j=1}^k (1 \pm \frac{1}{n^C}) \frac{\mu_{i_j}}{1-\sum_{j' < j} \mu_{i_{j'}}} = (1 \pm \frac{2k}{n^C}) \prod_{j=1}^k \frac{\mu_{i_j}}{1-\sum_{j' < j} \mu_{i_{j'}}},$$
for an arbitrarily large constant $C > 0$. 
Summing up over all orderings, we obtain the probability of obtaining
the sample $\{i_1, \ldots, i_k\}$ is within $(1 \pm \frac{1}{n^C})$ times
its probability of being sampled from $\mu$ in a sample without
replacement of size $k$, where $C > 0$ is a sufficiently large constant.

It remains to show $\Pr[E] \leq n^{-C}$ for an arbitrarily large
constant $C > 0$. Here we use that for all $j \in \{1, 2, \ldots, n\},$
$R(j) = x_i \pm \left (\frac{1}{2k} \right ) ^{1/p} \cdot \|\tail_k(x)\|_p.$
Let $Y_i$ be the number of trials until (and including the time) we sample the $i$-th distinct item, given
that we have just sampled $i-1$ distinct items. The total probability mass on the items we have already sampled is at most $i \cdot \frac{1}{2k} \|\tail_k(x)\|_p^p$, and thus the probability we re-sample an item already sampled is at most $\frac{1}{2}$. It follows that ${\bf E}[Y_i] \leq 2$. Thus, the number of trials in the algorithm is stochastically dominated by $\sum_{i=1}^k Z_i$, where $Z_i$ is a geometric random variable with ${\bf E}[Z_i] = 2$. This sum is a negative binomial random variable, and by standard tail bounds relating a sum of independent geometric random variables to binomial random variables\cite{mr99} \footnote{See, also, e.g., \url{https://math.stackexchange.com/questions/1565559/tail-bound-for-sum-of-geometric-random-variables}}, is at most $Ck \log n$ with probability $1-\frac{1}{n^C}$ for an arbitrarily large constant $C > 0$. 

This completes the proof. 
\end{proof}

We now analyze the memory in Theorem \ref{thm:tvd} more precisely. Algorithm \ref{lowTVD:alg} runs $r = O(k \log n)$ independent perfect $\ell_p$-sampling algorithms of \cite{JayaramW:Focs2018}. The choice of $r = O(k \log n)$ is to ensure that the variation distance is at most $\frac{1}{\poly(n)}$; however, with only $r = O(k)$ such samplers, the same argument as in the proof of Theorem \ref{thm:tvd} gives variation distance at most $2^{-\Theta(k)}$. Now, each $\ell_p$-sampler of \cite{JayaramW:Focs2018} uses $O(\log^2 n (\log \log n)^2)$ bits of memory for $0 < p < 2$, and uses $O(\log^3 n)$ bits of memory for $p = 2$. We also do not need to repeat the algorithm $O(\log n)$ times to create a high probability of not outputting {\sc Fail}; indeed, already if with only constant probability the algorithm does not output {\sc Fail}, we will still obtain $k$ distinct samples with $2^{-\Theta(k)}$ failure probability provided we have a large enough $r = O(k)$ number of such samplers. 

Algorithm \ref{lowTVD:alg} also runs an $\ell_p$ rHH method, and this uses $O(k \log^2 n)$ bits of memory \cite{Jowhari:Saglam:Tardos:11}. Consequently, to acheive variation distance at most $2^{-\Theta(k)}$, Algorithm \ref{lowTVD:alg} uses $O(k \log^2 n (\log \log n)^2)$ bits of memory for $0 < p < 2$, and $O(k \log^3 n)$ bits of memory for $p = 2$.

We now show that for $0 < p < 2$, the memory used of Algorithm \ref{lowTVD:alg} is best possible for any algorithm, up to a multiplicative $O((\log \log n)^2)$ factor. For $p = 2$, we show our algorithm's memory is optimal up to a multiplicative $O(\log n)$ factor. Further, our lower bound holds even for any algorithm with the much weaker requirement of achieving variation distance at most $\frac{1}{3}$, as opposed to the variation distance at most $2^{-\Theta(k)}$ that we achieve. 

\begin{theorem}
Any $1$-pass turnstile streaming algorithm which outputs a set $S$ of size $k$ such that the distribution of $S$ has variation distance at most $\frac{1}{3}$ from the distribution of a sample without replacement of size $k$ from the distribution $\mu = (\mu_1, \ldots, \mu_n)$, where $\mu_i = \frac{|x_i|^p}{\|x\|_p^p}$, requires $\Omega(k \log^2 n)$ bits of memory, provided $k < n^{C_0}$ for a certain absolute constant $C_0 > 0$. 
\end{theorem}
\begin{proof}
We use the fact that such a sample $S$ can be used to find a constant fraction of the $\ell_q (k,1)$ residual heavy hitters in a data stream. Here we do not need to achieve residual error for our lower bound, and can instead define such indices $i$ to be those that satisfy $|x_i|^p \geq \frac{1}{k} \|x\|_p^p$. Notice that there are at most $k$ such indices $i$, and any sample $S$ (with or without replacement) with variation distance at most $1/3$ from a true sample has probability at least $1-(1-1/k)^k -1/3 \geq 1-1/e-1/3 \geq .29$ of containing the index $i$. By repeating our algorithm $O(1)$ times, we obtain a superset of size $O(k)$ which contains any particular such index $i$ with arbitrarily large constant probability, and these $O(1)$ repetitions only increase our memory by a constant factor.

It is also well-known that there exists a point-query data structure, in particular the \texttt{CountSketch} data structure \cite{ccf:icalp2002,price11}, which only needs $O(\log |S|) = O(\log k)$ rows and thus $O((k \log k) \log n)$ bits of memory, such that given all the indices $j$ in a set $S$, one can return all items $j \in S$ for which $|x_j|^p \geq \frac{1}{k} \|x\|_p^p$ and no items $j \in S$ for which $|x_j|^p < \frac{1}{2k} \|x\|_p^p$. Here we avoid the need for $O(\log n)$ rows since we only need to union bound over correct estimates in the set $S$.

In short, the above procedure allows us to, with arbitrarily large constant probability, return a set $S$ containing a random $.99$ fraction of the indices $j$ for which $|x_j|^p \geq \frac{1}{k} \|x\|_p^p$, and containing no index $j$ for which $|x_j|^p < \frac{1}{2k} \|x\|_p^p$. 

We now use an existing $\Omega(k \log^2 n)$ bit lower bound, which is stated for finding all the heavy hitters \cite{Jowhari:Saglam:Tardos:11}, to show an $\Omega(k \log^2 n)$ bit lower bound for the above task. This is in fact immediate from the proof of Theorem 9 of \cite{Jowhari:Saglam:Tardos:11}, which is a reduction from the one-way communication complexity of the Augmented Indexing Problem and just requires any particular heavy hitter index to be output with constant probability. In particular, our algorithm, combined with the $O((k \log k) \log n)$ bits of memory side data structure of \cite{ccf:icalp2002} described above, achieves this.  

Consequently, the memory required of any $1$-pass streaming algorithm for the sampling problem is at least $\Omega(k \log^2 n) - O((k \log k) \log n)$ bits, which gives us an $\Omega(k \log^2 n)$ lower bound provided $k < n^{C_0}$ for an absolute constant $C_0 > 0$, completing the proof.  
\end{proof}

\section{Estimates of one-pass WORp} \label{onepassquality:sec}

We first review the setup.  Our one-pass WORp method returns the top $k$ keys by $\widehat{\nu^*_x}$ as our sample $S$ and returns
$\widehat{\nu^*}_{(k+1)}$ as the threshold.
The estimate of $f(\nu_x)$ is $0$ for $x\not\in S$ and for $x\in S$ is
\begin{equation} 
\widehat{f(\nu_x)} := \frac{ f(\widehat{\nu^*_x} r_x^{1/p})}{1- \exp\left(- r_x (\frac{\widehat{\nu^*_x}}{\widehat{\nu^*}_{(k'+1)}})^p\right)}\enspace .
\end{equation}

We assume that $f(w)$ is such that for some constant $c$, 
\begin{equation} \label{fprop:eq}
\forall \varepsilon<1/2, 
|f((1+\varepsilon)w)-f(w)| \leq  c \varepsilon f(w)\enspace .
\end{equation}

We need to establish that
\begin{align*}
    \Bias[\widehat{f(\nu_x)}] &\leq  O(\varepsilon) f(\nu_x)\\
   \MSE[\widehat{f(\nu_x)}] &\leq   (1+O(\varepsilon)) \Var[\widehat{f(\nu_x)}'] + O(\varepsilon) f(\nu_x)^2 \enspace , 
\end{align*}
  where 
$\widehat{f(\nu_x)}'$ are estimates obtained with a (perfect) $p$-ppswor sample.

  
\begin{proof} [Proof of Theorem~\ref{onepassquality:thm}]
From \eqref{usePsi:eq}, the rHH sketch has the property that for all keys in the dataset,
\begin{equation}\label{err1p:eq}
\| \widehat{\bnu^*} -\bnu^* \|_\infty \leq \varepsilon \nu^*_{(k+1)} \enspace .
\end{equation}
For sampled keys, $|\nu^*_x| \geq |\nu^*_{(k+1)}|$ and hence 
$|\widehat{\nu^*_x} - \nu^*_x| \leq \varepsilon |\nu^*_{(k+1)}| \leq \varepsilon |\nu^*_x|$.  Using \eqref{out2in:eq} we obtain that 
$\|\nu'_x - \nu_x\| \leq \varepsilon |\nu_x|$.

From our assumption \eqref{fprop:eq} , we have
$|f(\nu'_x)-f(\nu_x)|\leq c \varepsilon f(\nu_x)$.

We consider the inclusion probability and frequency estimate of a particular key $x$, conditioned on fixed randomization $r_z$ of all other keys $z\not= x$.
The key $x$ is included in the sample if $\widehat{\nu^*_x} \geq \widehat{\bnu^*}_{(k+1)}$.
We consider the distribution of $\widehat{\nu^*_x}$ as a function of $r_x\sim \Exp[1]$.
The value has a form of $E + \nu_x/r_x^{1/p}$, where the erro $E$ 
satisfies $|E| \leq \varepsilon |\nu^*_{(k+1)}|$.  
The conditioned inclusion probability thus falls in the range
\[
p'_x = \Pr[\nu_x/r_x^{1/p} \pm \varepsilon |\nu^*_{(k)}| \geq \widehat{\bnu^*}_{(k)}] = \Pr\left[ r_x \leq \left(\frac{\nu_x}{\widehat{\bnu^*}_{(k+1)} \pm \varepsilon |\nu^*_{(k)}|}\right)^p\right] = 1-\exp(- \left(\frac{\nu_x}{\widehat{\bnu^*}_{(k+1)} \pm \varepsilon |\nu^*_{(k)}|}\right)^p)\enspace .
\]
We estimate $p'_x$ by 
\begin{equation}\label{pestapprox:eq} 
p''_x = 1- \exp\left(- r_x (\frac{\widehat{\nu^*_x}}{\widehat{\nu^*}_{(k+1)}})^p\right) \enspace . 
\end{equation}
This estimate
has a small relative error. This due to the relative error in $\widehat{\nu^*_x}$ and
because  $|(1-\exp(-(1\pm \epsilon) b)))- (1-\exp(-b))| = O(\epsilon) (1-\exp(-b))$ and
$(\frac{\nu'_x}{\widehat{\bnu^*}_{(k)}})^p)$ is an $O(\epsilon)$ relative error approximation of
$(\frac{\nu_x}{\widehat{\bnu^*}_{(k)} -E})^p$.

We first consider the bias.  Instead of using the unbiased inverse probability estimate $f(\nu_x)/p'_x$ when $x$ is sampled (with probability $p'_x$) our estimator~\eqref{xppsworest:eq} $(f(\nu'_x)/p''_x$ approximates both the numerator and the denominator.  

In the numerator of the estimator,  we replace $f(\nu_x)$ by the relative error approximation  $f(\nu'_x)$. Therefore overall, we use a small relative error estimate of the actual inverse probability estimate when it is non zero, which translates to a bias that is $O(\epsilon) f(\nu_x)$.

We next bound the Mean Squared Error (MSE) of the estimator~\eqref{xppsworest:eq}.  We express the variance contribution of
exact $p$-ppswor conditioned on the same randomization $r_z$ of all keys $z\not=x$. 
This is  $\Var[\widehat{f(\nu_x)}'] = (1/p_x -1) f(\nu_x)^2$, where
$p_x = \Pr[\nu_x/r_x^{1/p} \geq \bnu^*_{(k)}] = 1-\exp(- \left(\frac{\nu_x}{\bnu^*_{(k)}}\right)^p)$.
The MSE contribution is
\begin{equation} \label{approxMSE:eq}
    p'_x(f(\nu'_x)/p''_x -f(\nu_x))^2 + (1-p'_x)f(\nu_x)^2 \enspace .
\end{equation}

We observe that the approximate threshold (that determines $p'_x$) approximates the perfect $p$-ppswor threshold:
$| \widehat{\bnu^*}_{(k)}-\bnu^*_{(k)} | \leq \varepsilon |\nu^*_{(k)}|$.
When $p_x<1/2$, \eqref{approxMSE:eq} approximates $\Var[\widehat{f(\nu_x)}']$ with relative error $O(\varepsilon)$. 

When $p_x$ is close to $1$ this is dominated by $O(\varepsilon) f(\nu_x)^2$.

\end{proof}

We remark that our analysis of the error only assumed the rHH error bound \eqref{err1p:eq} which holds for all sketch types including \texttt{Counters}. The bias analysis can be tightened for \texttt{CountSketch} that returns unbiased estimates of the frequency.

\section{Pseudocode}

 \begin{algorithm2e}\caption{2-pass WORp}\label{twopass:alg}
 \KwIn{$\ell_q$ rHH method, sample size $k$, $p$, $\delta$, $n$, }
 \DontPrintSemicolon
 {\bf Initialization:}\;
 Draw a random hash $r_x \sim \mathcal{D}$\tcp*{Random map of keys $x$ to $r_x$}\;
 $\psi \gets \frac{1}{3}\Psi_{n,k,q/p}(\delta)$\;
 Initialize $\keyhash$ \tcp*{random hash function from strings to $[n]$}\;
 $R.\Rinit(k,\psi)$ \tcp*{Initialize rHH structure randomization }
 {\bf Pass I:} \tcp*{Use composable aggregation (process input keys into rHH structures and merge rHH structures)}
 \Begin{
    {\bf Process data element} $e=(e.\ekey,e.\eval)$ into rHH sketch $R$\;
  $R.\Rprocess(\keyhash(e.\ekey),e.\eval/r_{e.\ekey}^{1/p})$\tcp*{Generate and process output element}
 }
 {\bf Pass II:}\tcp*{For keys with top $2k$ estimates $\widehat{\nu^*_x}$, collect exact frequencies $\nu_x$.}\;
 Initialize a composable top-$2k$ structure $T$. The structure stores for each key its priority and frequency. The structure collects exact frequencies for the keys with top $2k$ priorities. 
 $\Rmerge(T_1,T_2)$: Add up values and retain $3k$ top priority keys.\;  

 {\bf Process data element} $e=(e.\ekey,e.\eval)$ into $T$\;
\Begin{
    \eIf{$e.\ekey \in T$}{$T[e.\ekey].val += e.\eval$}
    { 
    $est \gets R.\Rest[\keyhash(e.\ekey)]$ \tcp*{ $\widehat{\nu^*_{\ekey}}$} \; 
    \If{$est > \text{lowest priority in $T$}$}{Insert $e.\ekey$ to $T$\;
    $T[e.\ekey].val \gets e.\eval$\;
    $T[e.\ekey].priority \gets est$\; 
    \If{|T|>2k}{Eject lowest priority key from $T$\;}
    }
    }
    }
  {\bf Produce sample:}
  Sort $T$ by $T[x].val * r^{1/p}_x$ \tcp*{actual $\nu^*_x$ for keys in $T$}\;
  Return $(x,T[x].val)$ for top-$k$ keys and $(k+1)$th $T[x].val * r^{1/p}_x$\;
 \end{algorithm2e}

\end{document}